\documentclass[preprint,12pt]{elsarticle}
\usepackage{amssymb}
\usepackage{hyperref}
\usepackage{times}

\usepackage{amscd,amsmath,amssymb,amsthm}
\usepackage{color}

\usepackage{mathrsfs}
\usepackage{graphicx}

\usepackage{helvet}
\usepackage{courier}

\usepackage{algorithm}
\usepackage[noend]{algorithmic}
\usepackage{verbatim}
\usepackage{flafter}

\theoremstyle{definition}
\newtheorem{dfn}{Definition}
\theoremstyle{plain}
\newtheorem{prop}{Proposition}
\newtheorem{thm}{Theorem}
\newtheorem*{thm-other}{Theorem}

\newtheorem{coro}{Corollary}
\newtheorem{lemma}{Lemma}
\theoremstyle{remark}
\newtheorem{remark}{Remark}
\newtheorem{example}{Example}

\newtheoremstyle{example_contd}
{\topsep} {\topsep}%
{\upshape}% Body font
{}% Indent amount (empty = no indent, \parindent = para indent)
{\bfseries}% Thm head font
{}% Punctuation after thm head
{\newline}% Space after thm head (\newline = linebreak)
{\thmname{#1} \thmnumber{ #2}\thmnote{#3}\enspace(continued)}% Thm head spec

\theoremstyle{example_contd}

\newcommand{\mbr}{\mathcal{M}}
\newcommand{\dir}{\textsf{dir}}
\newcommand{\sfp}{\textsf{p}}
\newcommand{\sfm}{\textsf{m}}
\newcommand{\sfo}{\textsf{o}}
\newcommand{\sfd}{\textsf{d}}
\newcommand{\sfs}{\textsf{s}}
\newcommand{\sff}{\textsf{f}}
\newcommand{\sfeq}{\textsf{eq}}
\newcommand{\sffi}{\textsf{fi}}
\newcommand{\sfsi}{\textsf{si}}
\newcommand{\sfoi}{\textsf{oi}}
\newcommand{\sfdi}{\textsf{di}}
\newcommand{\sfmi}{\textsf{mi}}
\newcommand{\sfpi}{\textsf{pi}}

\newcommand{\net}{\mathcal{N}}
\newcommand{\sa}{\mathfrak{a}}

\newcommand{\true}{\textsf{true}}
\newcommand{\false}{\textsf{false}}

\newcommand{\consistency}{\textsc{Consistency}}
\newcommand{\cspan}{\textsc{Cspan}}

\newcommand{\rsat}{\textbf{RSAT}}

\begin{document}

\begin{frontmatter}


\title{Reasoning about Cardinal Directions between Extended Objects: The Hardness Result}


\author{Weiming Liu}
\author{Sanjiang Li\corref{cor1}}
\ead{sanjiang.li@uts.edu.au}
\address{Centre for Quantum Computation and Intelligent Systems,
       Faculty of Engineering and Information Technology, University of Technology,
       Sydney, Australia}
\cortext[cor1]{Corresponding author at:
Centre for Quantum Computation and Intelligent Systems, Faculty of
Engineering and Information Technology, University of Technology
Sydney, P.O. Box 123, Broadway, NSW 2007, Australia
}

\begin{abstract}
The cardinal direction calculus (CDC) proposed by Goyal and Egenhofer is a very expressive qualitative calculus for directional information of extended objects. Early work has shown that  consistency checking of \textit{complete} networks of basic CDC constraints is tractable while reasoning with the CDC in general is NP-hard. This paper shows, however, if allowing some constraints unspecified, then consistency checking of \textit{possibly incomplete} networks of basic CDC constraints is already intractable.  This draws a sharp boundary between the tractable and intractable subclasses of the CDC. The result is achieved by a reduction from the well-known 3-SAT problem.
\end{abstract}


\end{frontmatter}

\section{Introduction}
Direction relations between extended spatial objects are important commonsense knowledge.
Most existing direction relation models approximate a spatial object by a point (e.g. its centroid) or a box. This is certainly imprecise in real-world applications such as describing the directional information between two countries, say, Portugal and Spain \cite{SkiadopoulosK04}.

Goyal and Egenhofer \cite{GoyalE97,Goyal2000} proposed  the \textit{direction relation matrix} (DRM) for representing direction relations
between \textit{connected} plane regions. The original description of the DRM lacks formality and does not consider limit cases. This problem was fixed in \cite{SkiadopoulosK04}, where the model is called the \textit{cardinal direction calculus} (CDC).
When representing the direction of the primary object to a reference object, the CDC approximates the reference object by a box, while leaving the primary object unaltered. Therefore, the exact geometry of the primary object is used to a certain extent in the representation of the direction. This makes the CDC very expressive. As a matter of fact, this
calculus has 218 basic relations, each of which represents some \textit{definite} directional information between objects. Non-basic relations, which are unions of basic relations, represent \textit{indefinite} directional information between objects.

The CDC as a qualitative calculus is unlike other well-known qualitative calculi such as the Interval Algebra (IA) \cite{Allen83} and RCC8 \cite{RandellCC92}. The identity relation is not a CDC relation, but is contained in a unique basic CDC relation. The CDC is closed under neither converse nor composition. This means, the converse of a basic CDC relation (or the composition of two basic CDC relations) may be not a CDC relation, i.e. it may be not the union of some basic CDC relations \cite{GoyalE97,CiceroneF04,SkiadopoulosK04,Liu+2010}.

Consistency checking is the central reasoning problem in the CDC (and any other
qualitative calculus). Given a \textit{complete} network of CDC constraints
\begin{equation}\label{eq:sat}
\mathcal{N}=\{v_i\delta_{ij} v_j\}_{i,j=1}^n\ \ \mbox{(each}\
\delta_{ij}\ \mbox{is\ a\ CDC\ relation)}
\end{equation}
over $n$ spatial variables $v_1,\cdots,v_n$, we say $\mathcal{N}$ is \textit{consistent} (or \textit{satisfiable}) if there exist $n$ \emph{connected} plane regions $a_1,\cdots,a_n$ such that $(a_i,a_j)$ is an instance of $\delta_{ij}$ for any $1\leq i,j\leq n$. If the relations $\delta_{ij}$ are all  taken from a subclass $\mathcal{S}$ of the CDC, we write $\rsat(\mathcal{S})$ for the consistency decision problem restricted to $\mathcal{S}$. In particular, \rsat(CDC) denotes the consistency decision problem in the CDC.

To solve the general consistency problem of a qualitative calculus, an often used approach is to devise local consistency algorithms to completely solve
the decision problem over a subclass $\mathcal{S}$ of the calculus, which contains all basic relation, and then use backtracking method to solve the whole
decision problem. We call a complete network $\mathcal{N}$ of basic constraints $k$-\textit{consistent} if all subnetworks of $\mathcal{N}$ that involve $k$
variables are consistent.\footnote{Note this notion of $k$-consistency is different from the usual one defined on a finite universe \cite{Dechter03}.}
In particular, for a complete network of basic constraints, 3-consistency is equivalent to \textit{path-consistency} (cf. \cite{LiW06} for detailed discussion).

For the IA and RCC8, it is known that path-consistency decides the consistency of complete basic networks. Examples (cf.  \cite[Example~9]{SkiadopoulosK04} and \cite[Example~4]{Liu+2010}) show that, however, local $k$-consistency, in particular path-consistency, is insufficient to determine the consistency of basic CDC constraints. This makes reasoning with the CDC a very difficult problem. For a long time, it is even not known if consistency checking in the CDC is decidable.

The consistency checking problems with the CDC and/or akin formalisms have been discussed in several literatures \cite{CiceroneF04,SkiadopoulosK04,SkiadopoulosK05,Liu+2010}. In particular, Liu et al. \cite{Liu+2010} provided a cubic
algorithm for checking the consistency of complete networks of basic CDC constraints, and proved that reasoning with the CDC in general is an
NP-Complete problem. This means $\rsat(\mathcal{B}_{dir})$ is tractable but $\rsat(\mbox{CDC})$ is not, where $\mathcal{B}_{dir}$ represents the set of all basic CDC relations. Before this work, we did not know if the CDC has larger tractable subclasses, not to mention finding maximal tractable subclasses and  determining the boundary between the tractable and intractable subclasses of the CDC.

To find maximal tractable subclasses of a qualitative calculus, an often used technique is to propagate the tractability of a subclass to its closure in the calculus under converse, intersection, and weak composition. This technique was first developed in reasoning with the IA \cite{NebelB95}, and then applied to reasoning with RCC8 \cite{RenzN99,LiW06} and general qualitative calculi in which path-consistency decides the consistency of a complete basic network \cite{Renz07}. Because the CDC does not have this property, the applicability of this technique is not immediately clear.

When discussing topological inference, Grigni et al. \cite{GrigniPP95} distinguished between two important special cases of constraint networks that are of interest: In the \textit{explicit} case, all constraints are basic (the relation for each pair of variables is specified). In the \textit{conjunctive}  case some constraints are basic while all the others are unspecified. This latter situation ``arises in geographic applications where the relation between objects in the same map is known, but not explicit information is given about objects in different maps. \cite{GrigniPP95}"

The consistency decision problem of explicit constraint networks corresponds to \rsat($\mathcal{B}_{dir}$), and that of conjunctive constraint networks corresponds to $\rsat(\mathcal{B}_{dir}\cup\{\ast\})$, where $\ast$ is the universal relation, i.e. the union of all basic relations. Having seen that $\mathcal{B}_{dir}$ is a tractable subclass of the CDC \cite{Liu+2010}, we are inclined to believe that $\mathcal{B}_{dir}\cup\{\ast\}$ is also a tractable subclass of the CDC. This paper, however, shows that this is not the case. Note that the universal relation is the weak composition of two basic CDC relations \cite{Liu+2010}. This suggests that the propagation technique used in e.g.  \cite{NebelB95,RenzN99,LiW06,Renz07} fails to find the maximal tractable subclasses of the CDC.

It seems that the CDC is the \emph{first} qualitative calculus in which reasoning with conjunctive constraints has different complexity as reasoning with explicit constraints.

We obtain the result by showing that there is a polynomial reduction from the 3-SAT problem to $\rsat(\mathcal{B}_{dir}\cup\{\ast\})$.
The reduction is devised based on the observation that some \emph{non-CDC} relations are \emph{definable} in the CDC (see Definition \ref{dfn:entailed-relation}). In particular, the \emph{upper left corner} (ULC) relation is defined by using only basic CDC constraints, where two bounded regions have the ULC relation if their minimum bounding rectangles (mbrs) are incomparable but have the same upper left corner point (see Figure \ref{fig:ulcillu} for illustrations). When considering only rectangles, the ULC relation is exactly the union of two basic rectangle relations, namely $\sfs\otimes\sffi$ and its converse $\sfsi\otimes\sff$, where $\sfs,\sff$ are basic relations in the IA, and $\sfsi$ and $\sffi$ are their converses (see Table \ref{tab:int} for the meanings of basic IA relations). Write $\ulcorner$ for the ULC relation.
The consistency decision problem  over $\mathcal{B}_{dir}\cup\{\ast\}$ is, roughly speaking, equivalent to that over $\mathcal{B}_{dir}\cup\{\ast,\ulcorner\}$. The NP-hardness of $\mathcal{B}_{dir}\cup\{\ast,\ulcorner\}$ is then not hard to imagine.

\begin{figure}[htb]
\centering
\begin{tabular}{ccc}
\includegraphics[width=.25\textwidth]{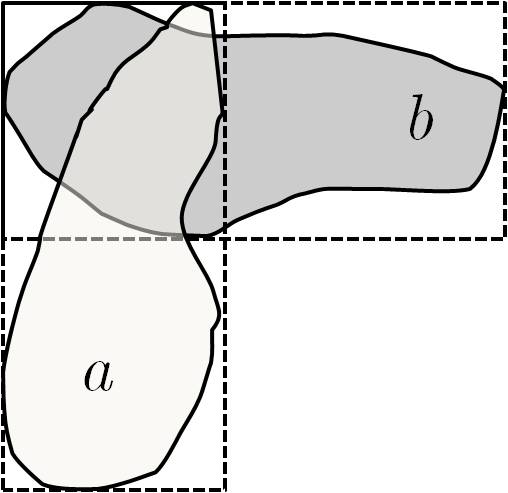}%[width=5in]
&
\includegraphics[width=.25\textwidth]{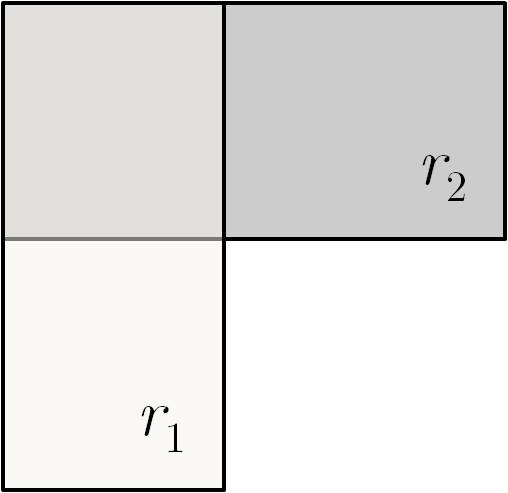}%[width=5in]
&
\includegraphics[width=.25\textwidth]{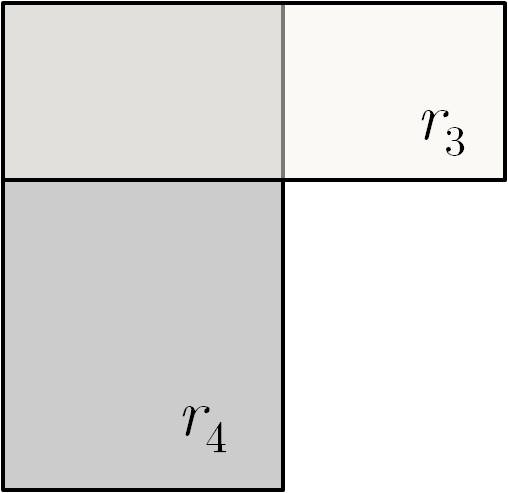}%[width=5in]
\\
(a) & (b) & (c)
 \end{tabular}
\caption{Illustrations of the symmetric ULC relation: (a) an instance $(a,b)$ of the ULC relation; (b) an instance $(r_1,r_2)$ of the RA relation $\sfs\otimes\sffi$; (c) an instance $(r_3,r_4)$ of the RA relation $\sfsi\otimes\sff$.}
\label{fig:ulcillu}
\end{figure}

Such a technique for defining relations outside a qualitative calculus was also used in \cite{KrokhinJJ03} for generalizing the tractability of subclasses of the IA.

Our reduction does not require regions to be connected. Therefore, the above NP-hardness result is also applicable to CDC$_d$, a variant of the CDC which deals with  cardinal direction relations between possibly disconnected plane regions \cite{SkiadopoulosK05,Liu+2010}. That is, the consistency decision problem of (possibly incomplete) basic CDC$_d$ networks is also an NP-hard problem. This suggests that the $O(n^5)$ consistency checking algorithm proposed in \cite{SkiadopoulosK05} is incomplete.

The remainder of this paper proceeds as follows. Section~2 introduces the CDC and some basic notions used in this paper. Section~3 shows examples of relations outside the CDC that are definable in the CDC. The main result is proved in Section~4, which is followed by an analysis of the correctness of the $O(n^5)$ algorithm in \cite{SkiadopoulosK05}. The last section concludes the paper.

Table~\ref{tab:notations} summarizes major and special notations used in this paper.
\begin{table}
\centering
\begin{tabular}{cc}
  Notation & Meaning \\
  \hline\hline
  $a,b$ & regions\\
  $I_x(a),I_y(a)$ & the $x$- and $y$-projections of $a$\\
  $\mbr(a)$ & the  minimum bounding rectangle (mbr) of $a$\\
  $\alpha,\beta$ & basic IA relations\\
  $\alpha\otimes\beta$ & a basic RA relation\\
  $\delta_1\!:\!\cdots\!:\!\delta_k$ & a basic CDC relation\\
  $u,v,w$ & spatial variables\\
  $\parallel$ & the right-side parallel relation with gap (\emph{parallel relation} for short)\\
  $\ulcorner$ & the upper left corner (ULC) relation\\
  $\net_{\alpha\otimes\beta}$ & the basic CDC network that entails the RA relation $\alpha\otimes\beta$\\
  $\net_\parallel$ & the basic CDC network that entails $\parallel$\\
  $\net_\ulcorner$ & the basic CDC network that entails $\ulcorner$\\
  $\phi$ & a 3-SAT instance\\
  $p,p_r,p_s,p_t$ & propositional variables\\
  $p_r^\ast,p_s^\ast,p_t^\ast$ & propositional literals\\
  $c,c_j$ & propositional clauses\\
  $\net_p$ & the basic CDC network for propositional variable $p$\\
  $\net_V$ & the basic CDC network for all propositional variables in $V$\\
  $\net_c$ & the basic CDC network for propositional clause $c$\\
  $\net_\phi$ & the basic CDC network for 3-SAT instance $\phi$\\
  $f_p,f_{\neg p},f_p^0$ & the frame spatial variables for propositional variable $p$\\
  $u_p,u_{\neg p}$ & the dual spatial variables for propositional variable $p$\\
  $v_c$ & the spatial variable for propositional clause $c$\\
  $w^c_0,w^c_{rs},w^c_{st},w^c_0$ & the `pier' spatial variables for propositional clause $c$\\
  $u^\ast_r$ & the spatial variable corresponding to $p^\ast_r$\\
  $X_c$ & the spatial variable set $\{w^c_0,u^\ast_r,w^c_{rs},u^\ast_s,w^c_{st},u^\ast_t,w^c_1\}$
\end{tabular}
\caption{Notations.}\label{tab:notations}
\end{table}

\section{Cardinal Direction Calculus: Definitions and Basic Notations}
In this section, we introduce definitions and basic notations of the cardinal direction calculus. We refer the readers to \cite{Liu+2010} for further discussions about this calculus.

The CDC is a calculus defined over connected plane regions. As usual, a region is defined as a nonempty regular closed subset of the plane. We assume all regions, if not stated otherwise, are bounded. We say a region is connected if it has a connected interior. Note a connected region may have disconnected exterior, i.e. it may have holes.

For a bounded set $b$ in the real plane, let
\begin{eqnarray}
x^-(b)=\inf\{x:(x,y)\in b\}, & x^+(b)=\sup\{x:(x,y)\in b\},\\
y^-(b)=\inf\{y:(x,y)\in b\}, & y^+(b)=\sup\{y:(x,y)\in b\}.
\end{eqnarray}
We write
\begin{equation}\label{eq:Ix(b)}
I_x(b)=[x^-(b),x^+(b)],\ \ \ I_y(b)=[y^-(b),y^+(b)].
\end{equation}
Let
\begin{equation}\label{eq:mbr}
\mbr(b)=I_x(b)\times I_y(b).
\end{equation}
We call $\mbr(b)$ the \emph{minimum bounding rectangle} (mbr) of
$b$, and call $I_x(b)$ and $I_y(b)$ {\label{+Ix(a)}} the $x$- and $y$-projection of
$b$, respectively. Clearly, $\mbr(b)$ is the smallest rectangle
which contains $b$ and has sides parallel to the axes.

\begin{figure}[htb]
\centering
\begin{tabular}{cc}
\hspace*{-10mm}
\includegraphics[width=.4\textwidth]{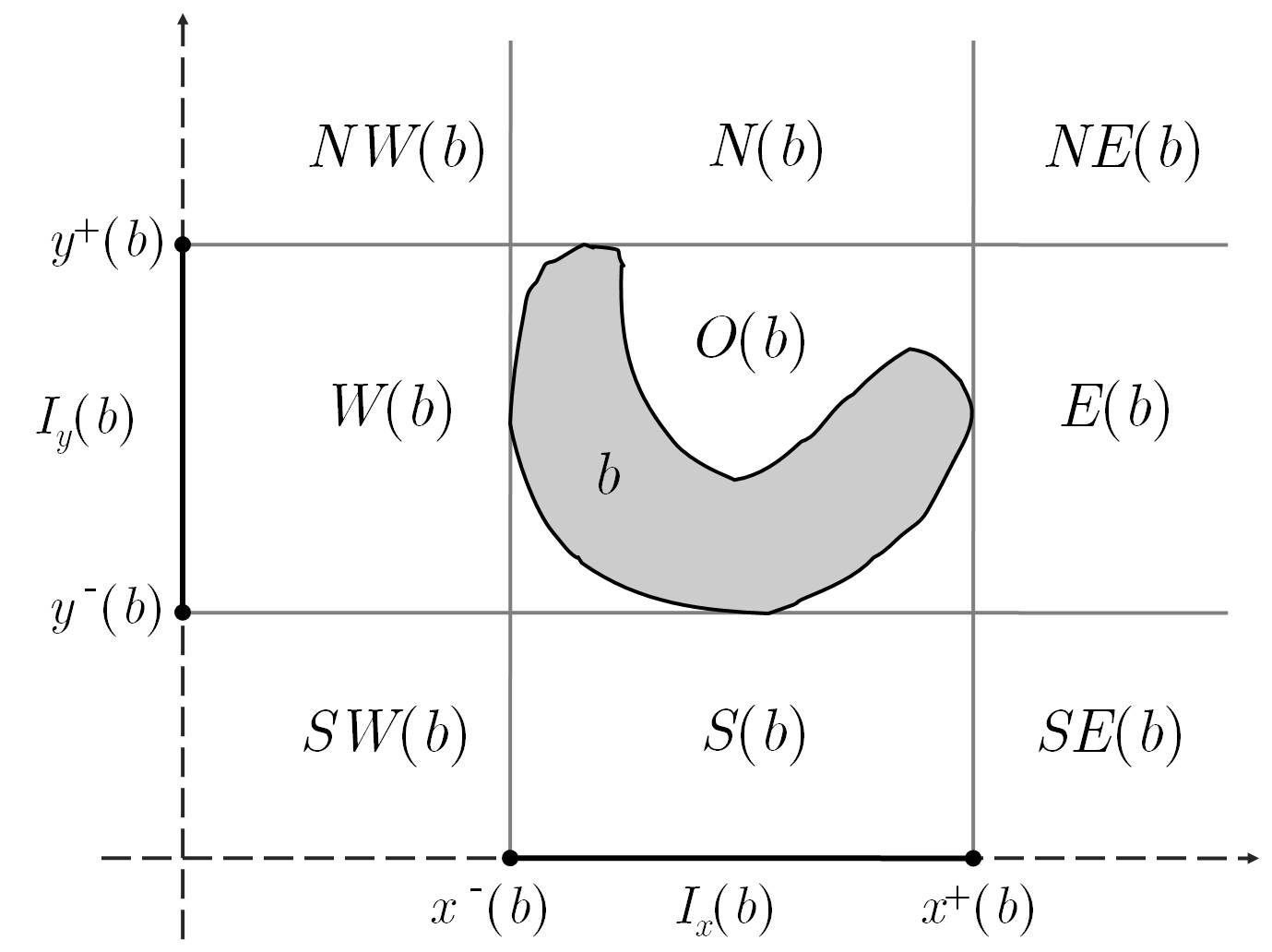}%[width=5in]
&
\includegraphics[width=.4\textwidth]{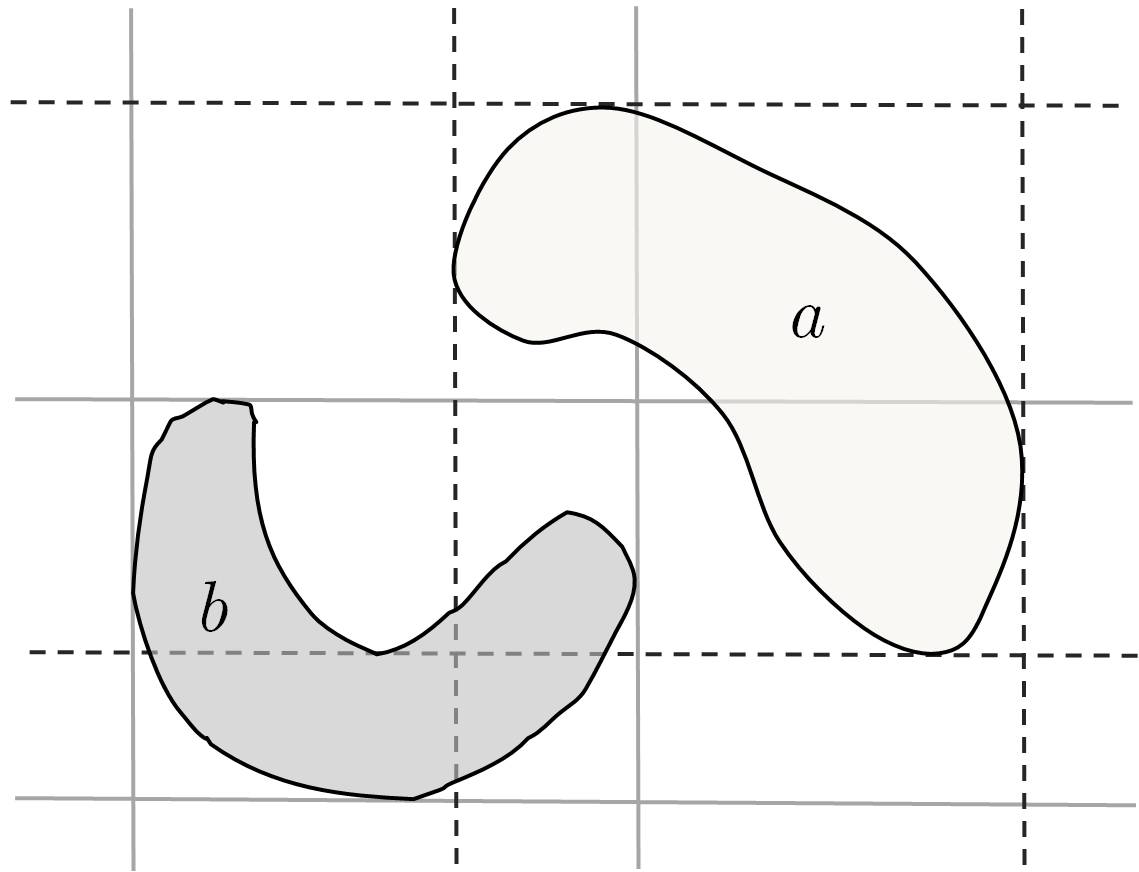}%[width=5in]
\\
(a) & (b)
 \end{tabular}
\caption{(a) A bounded connected region $b$ and its 9-tiles; (b) a pair of regions $a,b$.}
\label{fig:9tiles}
\end{figure}

By extending the four edges of $\mbr(b)$, we partition the plane into nine tiles, denoted as $NW(b),N(b),NE(b),W(b),O(b),E(b),SW(b)$,
$S(b),SE(b)$ (see Figure~\ref{fig:9tiles}(a)). Note that each tile is a (bounded or unbounded) connected region,
and the intersection of two tiles is of dimension lower than two.

The notion of direction relation matrix was first proposed
by Goyal and Egenhofer \cite{GoyalE97} for representing the
cardinal direction between extended spatial objects.

\begin{dfn}[direction relation matrix] \label{dfn:cdc}
Suppose $a,b$ are two bounded connected regions. Take $b$ as the
reference object, and $a$ as the primary object. The cardinal direction of
$a$ to $b$ is encoded in a $3\times 3$ Boolean matrix
\begin{equation}\label{eq:dir(a,b)}
\dir(a,b)=\left[
            \begin{array}{ccc}
              d^{NW} & d^N & d^{NE}\\
              d^{W} & d^O  & d^E \\
              d^{SW} & d^S & d^{SE} \\
            \end{array}
          \right],
\end{equation}
where for each tile name $\chi\in \{NW,N,NE,W,O,E,SW,S,SE\}$
\begin{equation}\label{eq:dchi}
d^\chi=1  \Leftrightarrow a^\circ\cap \chi(b)\not=\varnothing,
\end{equation}
where $a^\circ$ is the interior of $a$ and $\chi(b)$ denotes the $\chi$-tile of $b$. The cardinal direction relation of $a$ to $b$ is compactly represented in the form $\delta_1\!:\!\delta_2\!:\!\cdots\!:\!\delta_k$, where $\{\delta_1,\delta_2,\cdots,\delta_k\}$ is the set of tile names $\chi$ such that $d^\chi=1$.
\end{dfn}

Take the two regions $a,b$ in Figure~\ref{fig:9tiles}(b) as example. The cardinal direction relation of $a$ to $b$ is $N\!:\!NE\!:\!E$ and that of $b$ to $a$ is $W\!:\!O\!:\!SW\!:\!S$.

There are altogether 218 cardinal direction relations. Write $\mathcal{B}_{dir}$ for the set of these relations. Then each pair of connected regions are related by one and only one relation in $\mathcal{B}_{dir}$. This means that $\mathcal{B}_{dir}$ is a set of jointly exhaustive and pairwise disjoint (JEPD) relations. The cardinal direction calculus (CDC) is the Boolean algebra generated by $\mathcal{B}_{dir}$. Relations in $\mathcal{B}_{dir}$ are called basic relations of the CDC, and non-basic CDC relations are unions of basic relations. The universal relation, denoted by $\ast$, is in particular the union of all basic CDC relations.

\begin{remark}\label{remark:cdc_d}
Note that in the above definition we assume connected regions. When possibly disconnected regions are used, the calculus is called the cardinal direction calculus for possibly disconnected regions, written as CDC$_d$. There are 511 basic relations in CDC$_d$ \cite{SkiadopoulosK04,Liu+2010}.
\end{remark}

\section{Define Relations outside the CDC}
Non-basic CDC relations are formed by taking unions of basic relations. Relations outside the CDC may be \emph{definable} in the CDC in the following sense.

\begin{dfn}\label{dfn:entailed-relation}
Let $\net$ be a possibly incomplete network of basic CDC constraints over variables $u,v$ and $w_1,\cdots,w_k$. We say a relation $\gamma$ is \emph{entailed} by $\net$, written $\net\models \gamma$, if
\begin{equation}\label{eq:gamma}
\gamma=\{(a,b):(\exists c_1,\cdots,c_k)\ \mbox{s.t.}\ (a,b,c_1,\cdots,c_k)\ \mbox{is\ a\ solution\ of}\ \net\}.
\end{equation}
If $\gamma$ is entailed by some basic CDC network, then we also say $\gamma$ is \emph{definable} in the CDC.
\end{dfn}

An entailed relation is not necessarily a CDC relation.

The polynomial reduction from 3-SAT will rely heavily on one particular entailed relation, namely, the upper left corner (ULC) relation. Before discussing the ULC relation, we first review basic notions of the Rectangle Algebra and give some simple examples of entailed relations.

The Rectangle Algebra (RA) \cite{BalbianiCC99} is a qualitative calculus previously defined on all rectangles sides of which are parallel to the $x$- and $y$-axes. Relations in the RA can be naturally extended to the set of all bounded regions. For two bounded regions $a,b$, the basic RA relation of $a$ to $b$ is written as $\alpha\otimes\beta$, where $\alpha,\beta$ are basic IA relations and $I_x(a)\alpha I_x(b)$ and $I_y(a)\beta I_y(b)$. Table~\ref{tab:int} summarizes notations and definitions of the basic IA relations. For each basic RA relation $\alpha\otimes\beta$, the following equation is clear.
\begin{equation}\label{eq:RA-for-bounded-regions}
(a,b)\in \alpha \otimes \beta  \Leftrightarrow (\mbr(a),\mbr(b))\in \alpha\otimes\beta.
\end{equation}

\begin{table}[htb]\centering
\begin{tabular}{|c|c|c|c|}
  \hline
  Relation & Symbol & Converse & Meaning  \\ \hline
 before & \sfp & \sfpi & $x^-<x^+<y^-<y^+$  \\
 meets & \sfm & \sfmi &  $x^-<x^+=y^-<y^+$ \\
 overlaps & \sfo & \sfoi & $x^-<y^-<x^+<y^+$ \\
 starts & \sfs & \sfsi & $x^-=y^-<x^+<y^+$ \\
 during & \sfd & \sfdi & $y^-<x^-<x^+<y^+$ \\
 finishes & \sff & \sffi & $y^-<x^-<x^+=y^+$  \\
 equals & \sfeq & \sfeq & $x^-=y^-<x^+=y^+$ \\
  \hline
\end{tabular}
\caption{Basic IA relations and their converse, where
$x=[x^-,x^+],y=[y^-,y^+]$ are two intervals.}\label{tab:int}
\end{table}

We next show that some extended basic RA relations can be defined in the CDC.

\begin{example}\label{ex:net_sf}
A proper subset of the extended basic RA relation $\sfs\otimes\sff$ can be entailed by the following basic CDC network
\begin{equation}\label{eq:net_sf}
\net_{\sfs\otimes\sff}=\{u\ O\ v, v\ E\!:\!SE\!:\!S\!:\!O\ u\}.
\end{equation}
It is easy to see that if $(a,b)$ satisfies $\net_{\sfs\otimes\sff}$, then $\mbr(a)$ is contained in, and shares the upper left corner point with, $\mbr(b)$ (cf. Figure~\ref{fig:sf_para}(a)). In terms of the RA language, we have $(\mbr(a),\mbr(b))\in \sfs\otimes\sff$. We stress that $(a,b)$ may be not a solution to $\net_{\sfs\otimes\sff}$ even if $(\mbr(a),\mbr(b))\in \sfs\otimes\sff$. This is because the CDC relation of $b$ to $a$ could be, for example,  $E\!:\!SE\!:\!S$. But when only rectangles are considered, it is straightforward to see that $\sfs\otimes\sff$ is exactly the relation entailed by $\net_{\sfs\otimes\sff}$.

Similarly, we define
\begin{align}\label{eq:net_of}
\net_{\sfo\otimes\sff}&=\{u\ W\!:\!O\ v, v\ E\!:SE\!:\!S\!:\!O\ u\},\\ \label{eq:net_ofi}
\net_{\sfo\otimes\sffi}&=\{u\ S\!:\!SW\!:\!W\!:\!O\ v, v\ E\!:\!O\ u\},\\ \label{eq:net_oeq}
\net_{\sfo\otimes\sfeq}&=\{u\ W\!:\!O\ v, v\ E\!:\!O\ u\}.
\end{align}

\begin{figure}[htbp]
\centering
\begin{tabular}{cc}
  \includegraphics[width=.18\textwidth]{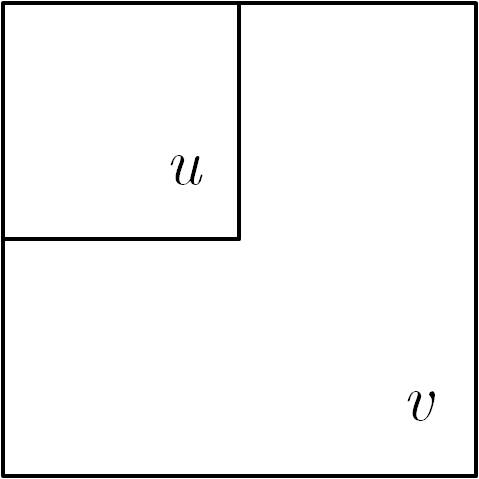} \hspace*{1cm}
  &
  \includegraphics[width=.36\textwidth]{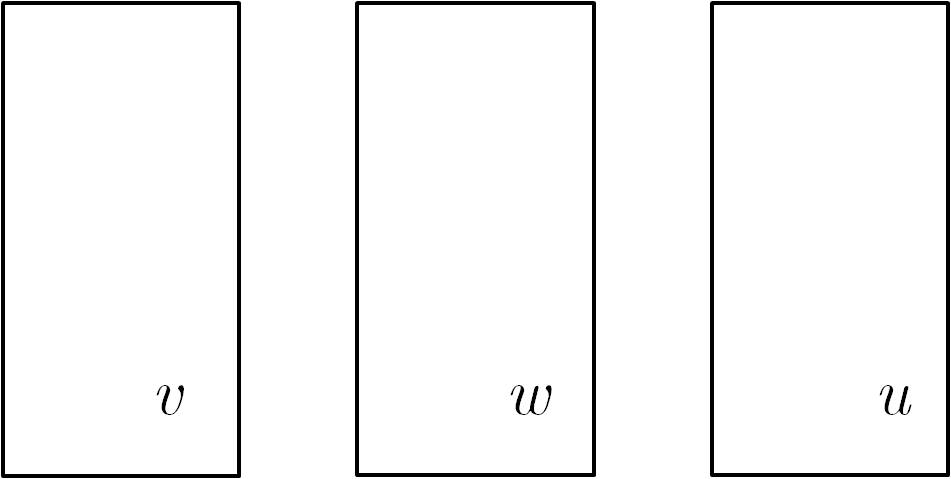}\\
  (a) & (b)
\end{tabular}
\caption{Illustrations of solutions for $\net_{\sfs\otimes\sff(u,v)}$ and $\net_{\parallel(u,v)}$}\label{fig:sf_para}
\end{figure}
\end{example}
We next introduce two entailed relations that involve auxiliary variables.
\begin{example}\label{ex:net_parallel}
We say $a$ is \emph{right-side parallel with gap} (or \emph{parallel} for short) to $b$,
if $I_x(a)\ \sfpi\ I_x(b)$ and $I_y(a)=I_y(b)$, i.e. $a$ is to the  east of $b$ (with gap) and has the same $y$-projection as $b$. This relation is entailed by the following basic CDC network
\begin{equation}\label{eq:net_parallel}
\net_{\parallel}=\{u\ E\ w, w\ E\ v, v\ W\ u\},
\end{equation}
where $w$ is an auxiliary variable (cf. Figure~\ref{fig:sf_para}(b)).
\end{example}

The parallel relation is strictly contained in the single tile relation $E$. Our next example is the upper left corner relation.
\begin{dfn}\label{dfn:upperleftcorner}
Two bounded plane regions $a,b$ are said to have the \emph{upper left corner (ULC) relation}, denoted as $\ulcorner(a,b)$, if the mbrs of $a,b$ are incomparable and have the same upper left corner point, or in the RA language, $(\mbr(a),\mbr(b))$ is an instance of either $\sfs\otimes\sffi$ or its converse $\sfsi\otimes\sff$ (cf. Figure~\ref{fig:ulcillu}).
\end{dfn}
The two possibilities of the ULC relation (cf. Figure \ref{fig:ulcillu}) correspond to the two truth values of a propositional variable. This fact will be exploited in the design of the polynomial reduction from 3-SAT. For convenience, we introduce the following terminologies.
\begin{dfn}\label{dfn:horizontal}
Suppose $\ulcorner(a,b)$. We say $a$ is \emph{horizontal} (\emph{vertical}, resp.) with respect to $b$, or $a$ is \emph{horizontally instantiated} (\emph{vertically instantiated}, resp.), if $\mbr(a)$ is related to $\mbr(b)$ by the RA relation $\sfsi\otimes\sff$($\sfs\otimes\sffi$ , resp.).
\end{dfn}

\begin{figure}\centering
\begin{tabular}{cc}
  \includegraphics[width=.22\textwidth]{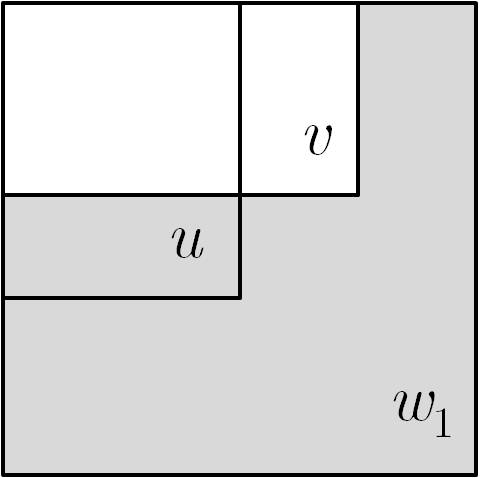} \qquad
  &
  \qquad \includegraphics[width=.22\textwidth]{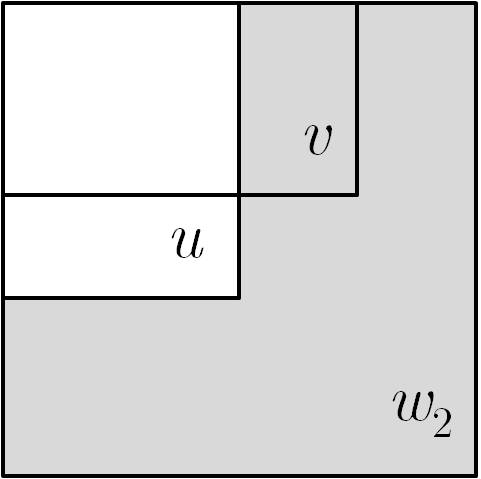}\\
  (a) \qquad & \qquad (b)
\end{tabular}
\caption{Illustrations of constraints in $\net_{\ulcorner}$ }\label{fig:ulc}
\end{figure}

The following proposition shows that the ULC relation can be defined in the CDC.
\begin{prop}
\label{lemma:1}
The ULC relation $\ulcorner$ can be entailed by basic CDC constraints.
\end{prop}
\begin{proof}
Two auxiliary variables $w_{1}$ and $w_{2}$ are introduced.
Let
\begin{equation}\label{eq:net_ulc}
\begin{split}
\net_{\ulcorner}=\{&u\ O\ w_{1}, w_1\ E\!:\!SE\!:\!S\!:\!O\ u,v\ O\ w_1,w_1\ E\!:\!SE\!:\!S\ v,\\
& v\ O\ w_2, w_2\ E\!:\!SE\!:\!S\!:\!O\ v,u\ O\ w_2,w_2\ E\!:\!SE\!:\!S\ u\}
\end{split}
\end{equation}
A basic constraint is imposed to each pair of variables in $\{u,v\}\times\{w_1,w_2\}\cup\{w_1,w_2\}\times \{u,v\}$. In particular, the constraints involving $w_1$ are $u\ O\ w_1$, $w_1\ E\!:\!SE\!:\!S\!:\!O\ u$, and $v\ O\ w_1$, $w_1\ E\!:\!SE\!:\!S\ v$. It is clear that these constraints imply $\mbr(u)$ and $\mbr(v)$ have the same upper left corner point as $\mbr(w_1)$ does (see Figure~\ref{fig:ulc}(a) for illustration). Similarly, $\mbr(u)$ and $\mbr(v)$ have the same upper left corner point as $\mbr(w_2)$ does (see Figure~\ref{fig:ulc}(b) for illustration). So the first requirement is satisfied. Furthermore, because $w_1\ E\!:\!SE\!:\!S\!:\!O\ u$ but $w_1\ E\!:\!SE\!:\!S\ v$, we know $\mbr(u)$ is not contained in $\mbr(v)$. Similarly, we have $\mbr(v)$ is not contained in $\mbr(u)$. Therefore, the second requirement is also satisfied.

On the other hand, if $(a,b)$ is an instance of $\ulcorner$, then we can find $c_1,c_2$ such that $\{a,b,c_1,c_2\}$ is a solution of $\net_{\ulcorner}$.
\end{proof}

\section{Consistency Checking of Conjunctive CDC Constraints}

This section proves that consistency checking of possibly incomplete basic CDC network  is an NP-hard problem. We achieve this by reducing the 3-SAT problem to the consistency checking problem $\rsat(\mathcal{B}_{dir}\cup\{\ast\})$. For each 3-SAT instance $\phi$, we construct an incomplete basic CDC network $\net_\phi$ in polynomial time, and show that $\phi$ is satisfiable if and only if $\net_\phi$ is consistent.

In this section, we assume $V=\{p_1,p_2,\cdots,p_n\}$ is a set of propositional variables. Suppose $\phi=c_1\wedge c_2\wedge \cdots \wedge c_m$, where clause $c_j$ is of the form $p_r^\ast\vee p_s^\ast\vee p_t^\ast$ and $p_r^\ast,p_s^\ast,p_t^\ast$ are literals over $V$.
We introduce a basic CDC network $\net_p$ for each propositional variable $p$,  and then introduce a basic CDC network $\net_c$ for each clause $c$. The basic CDC network $\net_\phi$ is defined as the union of all $\net_{c_j}$ ($1\leq j\leq m$). Note when expressing the constraints in $\net_c$, for simplicity, we often use non-CDC constraints which are definable in the CDC. We stress that if such a constraint, e.g. $\sfs\otimes\sff(u,v)$, appears, we always assume that it is replaced by the basic CDC constraints that entail it.

\subsection{CDC Constraints Related to Propositional Variables}

For each propositional variable $p$, we introduce five spatial variables $u_p$, $u_{\neg p}$, $f_p$, $f_{\neg p}$, and $f^0_p$, and define a set $\net_p$ of basic CDC constraints. Shortly we will give examples to show that $\net_p$ is consistent and has a solution in which all the above five spatial variables are rectangles. So in the following informal description, we assume the five variables are all rectangles for simplicity. The network $\net_p$ will ensure the two requirements:
\begin{itemize}
  \item The configuration of $f_p,f_{\neg p},f^0_p$ is as shown in Figure~\ref{fig:var}(a);
  \item Suppose $f_p,f_{\neg p},f^0_p$ are predefined. The configuration of $u_p,u_{\neg p}$ has two possibilities (cf. Figure~\ref{fig:var}(b) and (c)).
\end{itemize}
The second condition is mainly achieved by the fact that $u_p$ is horizontal w.r.t. $f_p$ iff $u_{\neg p}$ is vertical w.r.t. $f_{\neg p}$, which is guaranteed by the following constraints (cf. Figure~\ref{fig:var}(b) and (c)):
\begin{itemize}
  \item $u_p$ is contained in $f_{\neg p}$ and has the ULC relation with $f_p$,
  \item $u_{\neg p}$ contains $f_p$, and is contained in $f_p^0$, and has the ULC relation with $f_{\neg p}$,
  \item $u_p$ and $u_{\neg p}$ have the ULC relation.
\end{itemize}

\begin{figure}[htbp]
\centering
\begin{tabular}{ccc}
  \includegraphics[width=.25\textwidth]{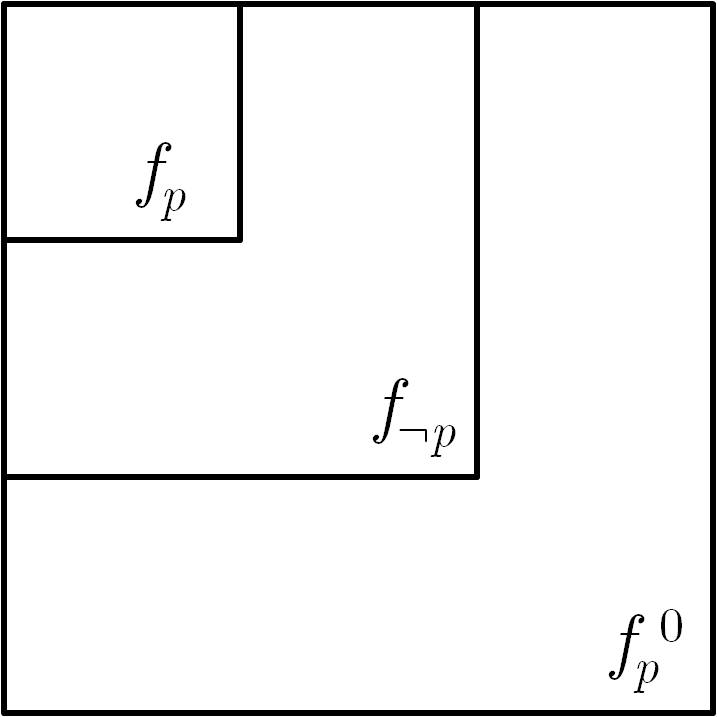} &
  \includegraphics[width=.25\textwidth]{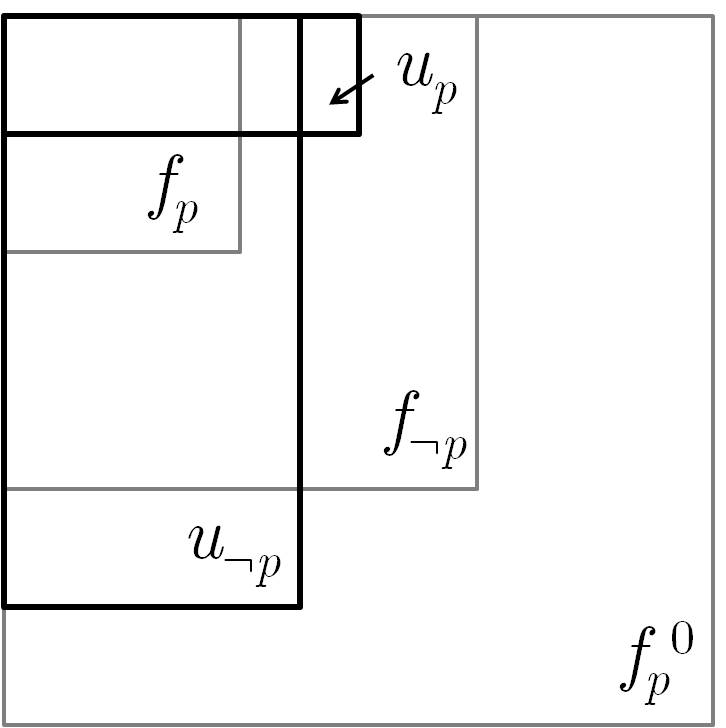} &
  \includegraphics[width=.25\textwidth]{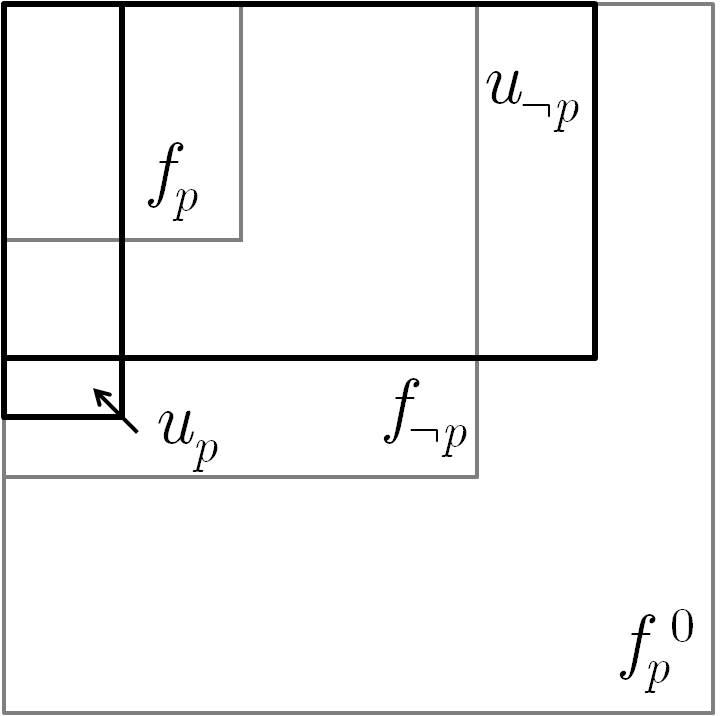} \\
(a) & (b) & (c)
\end{tabular}
\caption{Illustrations of spatial variables in $\{f_p,f_{\neg p},f_p^0, u_p,u_{\neg p}\}$: (a) the frame spatial variables $f_p,f_{\neg p},f_p^0$;
(b) a solution of $\net_p$ where $u_p$ is horizontally instantiated; (c) a solution of $\net_p$ where $u_p$ is vertically instantiated.}
\label{fig:var}
\end{figure}

The following definition specifies constraints in $\net_p$ formally.

\begin{dfn}\label{dfn:net-p}
Let $p$ be a propositional variable, and $u_p,u_{\neg p},f_p,f_{\neg p},f_p^0$ be five spatial variables. The basic CDC network $\net_p$ contains the basic CDC constraints $u_p\ O\ f_{\neg p}$ and $f_p\ O\ u_{\neg p}$, and the following non-CDC constraints
\begin{equation}\label{eq:other-constraints-in-net-p}
\ulcorner(u_p,f_p), \ulcorner(u_{\neg p},f_{\neg p}), \ulcorner(u_p,u_{\neg p}), \sfs\otimes\sff(f_p,f_{\neg p}), \sfs\otimes\sff(u_{\neg p},f_p^0), \sfs\otimes\sff(f_{\neg p},f_p^0),
\end{equation}
where $\ulcorner$ is defined in Eq.~\ref{eq:net_ulc} and $\sfs\otimes\sff$ is a basic RA relation. The non-CDC constraints appeared in $\net_p$ are replaced by the basic CDC constraints that entail them. We call $u_p$ and $u_{\neg p}$  the \emph{dual} spatial variables of $p$, and call $f_p$, $f_{\neg p}$, and $f^0_p$ the \emph{frame spatial variables} of $u_p$.
\end{dfn}

We note that except $u_p,u_{\neg p},f_p,f_{\neg p},f_p^0$, $\net_p$ also involves six other auxiliary spatial variables, which are introduced by the three ULC constraints.

Two solutions of $\net_p$ are shown in Figure~\ref{fig:var}, where $u_p$ is horizontally  instantiated in the solution shown in Figure~\ref{fig:var}(b), but vertically instantiated in the solution shown in Figure~\ref{fig:var}(c).
Though its position is non-determined, we know that the lower right corner of $\mbr(u_p)$ is in the interior of the shaded upper left (lower right, resp.) sub-rectangle if $u_p$ is horizontal (vertical, resp.) w.r.t. $f_p$ (see Figure~\ref{fig:var2}(a)). We next show, in any solution of $\net_p$, $u_p$ is vertical w.r.t. $f_p$ if and only if $u_{\neg p}$ is horizontal w.r.t. $f_{\neg p}$. Moreover, we show $\net_p$ has one solution in which $u_p$ is horizontal w.r.t. $f_p$  and another solution in which $u_p$ is vertical w.r.t. $f_p$.

\begin{figure}\centering
\begin{tabular}{ccc}
\includegraphics[width=.25\textwidth]{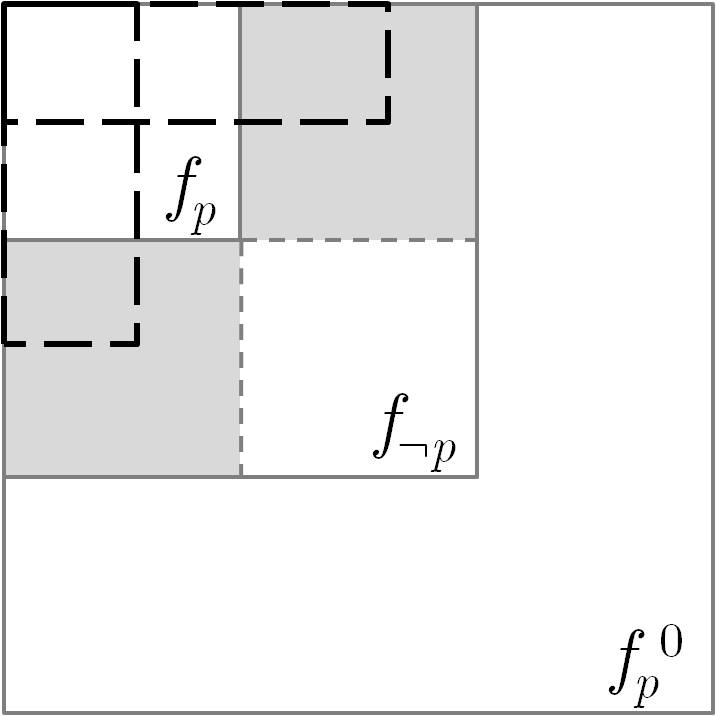} &
\includegraphics[width=.25\textwidth]{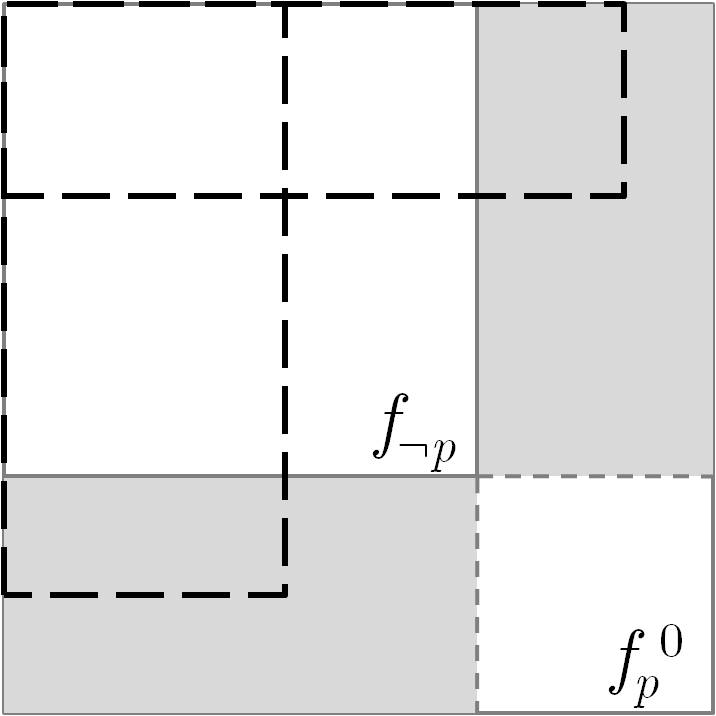} &
\includegraphics[width=.25\textwidth]{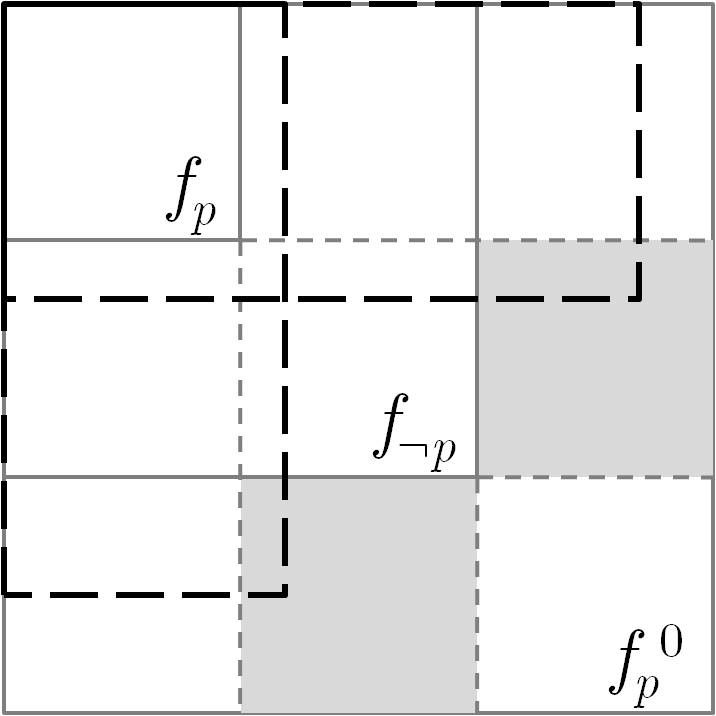} \\
(a) & (b) & (c)
\end{tabular}
\caption{Possible positions for the lower right corner points of $u_p$ (a) and $u_{\neg p}$ (b) (c)}
\label{fig:var2}
\end{figure}

\begin{prop}\label{lemma:2}
Let $\net_p$ be the basic CDC network of a propositional variable $p$. Suppose $\{u_p,u_{\neg p}, f_p, f_{\neg p}, f^0_p\}$ is a solution of $\net_p$. Then
$u_p$ is vertical w.r.t $f_p$ iff $u_{\neg p}$ is horizontal w.r.t. $f_{\neg p}$.
\end{prop}
\begin{proof}
It is clear that the mbrs of $u_p,u_{\neg p},f_p,f_{\neg p},f^0_p$ have the same upper left corner point.
We now consider the lower right corner points of $\mbr(u_p)$ and $\mbr(u_{\neg p})$. The constraints $\ulcorner(u_p,f_p)$ and $u_p\ O\ f_{\neg p}$ restrict the lower right corner point of $\mbr(u_p)$ to the shaded part in Figure~\ref{fig:var2}(a). Similarly, the constraints $\ulcorner(u_{\neg p},f_{\neg p})$ and $u_{\neg p}\ O\ f^0_p$ restrict the lower right corner point of $\mbr(u_{\neg p})$ to the shaded part in Figure~\ref{fig:var2}(b). By adding constraint $f_p\ O\ u_{\neg p}$, the possible area of the lower right corner point of $\mbr(u_{\neg p})$ is further restricted to the shaded part in Figure~\ref{fig:var2}(c).

Therefore, if $u_p$ and $u_{\neg p}$ are both vertical or both horizontal, then the lower right corner point of $\mbr(u_p)$ must be in $\mbr(u_{\neg p})$, which implies $\mbr(u_p)\subset\mbr(u_{\neg p})$ as their upper left corner points are the same. By imposing $\ulcorner(u_p, u_{\neg p})$ (which needs two extra auxiliary variables), $\mbr(u_p)$ and $\mbr(u_{\neg p})$ are necessarily partially overlapping, i.e., they can not be vertical or horizontal at the same time.
\end{proof}

The mutual exclusion of vertically and horizontally instantiations of dual variables $u_p,u_{\neg p}$ corresponds to the mutual exclusion of the truth values of  $p$ and its negation $\neg p$.

For each propositional variable $p_i\in V=\{p_1,p_2,\cdots,p_n\}$, we introduce a pair of dual spatial variables
\begin{equation}
u_i\equiv u_{p_i}\ \mbox{and}\ u_{\neg i}\equiv u_{\neg p_i},
\end{equation}
and three frame spatial variables
\begin{equation}
f_i\equiv f_{p_i}, f_{\neg i}\equiv f_{\neg p_i}\ \mbox{and}\ f^0_i\equiv f^0_{p_i},
\end{equation}
and construct, as described above in Definition~\ref{dfn:net-p}, a basic CDC network $\net_{p_i}$ over spatial variables $\{u_i,u_{\neg i},f_i,f_{\neg i},f_i^0\}$.

In order to fix the relative direction between two frame spatial variables of different propositional variables (cf. Figure~\ref{fig:frame}), we introduce a set of reference spatial variables. Precisely, let
\begin{eqnarray}
\label{eq:Vref}
V_{ref}=\{w_{ref},f_{ref},f_{\neg ref},f_{ref}^0\},\\
\label{eq:Nref}
\begin{split}
\net_{ref}=\{w_{ref}\ O\ f_{ref}\ O\ f_{\neg ref}\ O\ f_{ref}^0, \\
f_{ref}^0\ S\!:\!O\ f_{\neg ref}\ S\!:\!O\ f_{ref}\ S\!:\!O\ w_{ref}\},
\end{split}
\end{eqnarray}
where the shorthand, say, $x\ S\!:\!O\ y\ S\!:\!O\ z$ denotes that $x\ S\!:\!O\ y$ and $y\ S\!:\!O\ z$. Furthermore, we require
\begin{eqnarray}
\label{eq:f1-ref}
&&\parallel(f_1,f_{ref}),\  \parallel(f_{\neg 1},\  f_{\neg ref}),\ \parallel(f_1^0,f_{ref}^0),\\
\label{eq:parallel-constraints}
&&\parallel(f_{i+1},f_{i}),\ \parallel(f_{\neg(i+1)},f_{\neg i}),\ \parallel(f_{i+1}^0,f_{i}^0,), \hspace*{10mm} (1\leq i< n)
\end{eqnarray}
where the relation $\parallel$ is defined in Eq.~\ref{eq:net_parallel}.
Note these $3n$ parallel constraints introduce $3n$ new auxiliary variables.

\begin{figure}[htbp]
\centering
  \includegraphics[width=.9\textwidth]{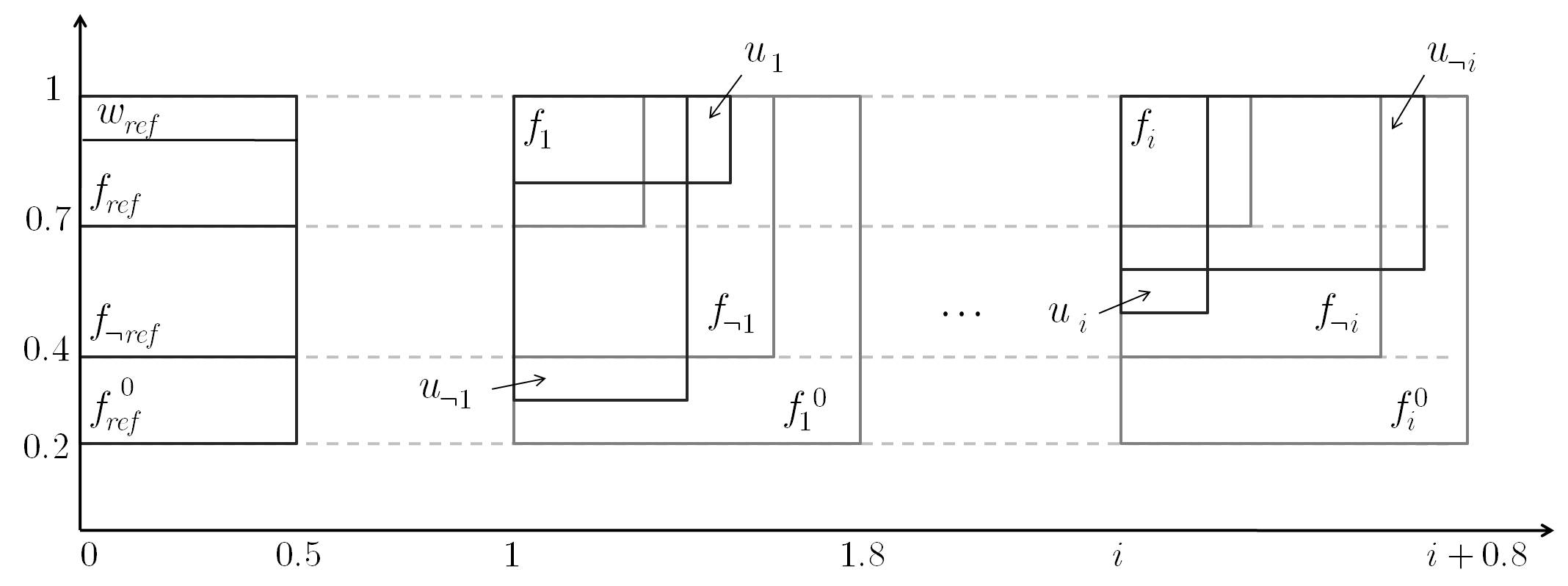}\\
  \caption{Illustration of a solution of $\net_V$}\label{fig:frame}
\end{figure}

The reference variable $w_{ref}$ in $V_{ref}$ will be used in the next subsection when constructing the basic CDC networks for propositional clauses.

\begin{dfn}\label{dfn:net_frame}
We write $\net_V$ for the set of basic CDC constraints that includes those in $\net_{ref}$ and  $\net_p$ for each $p\in V$, and those basic CDC constraints that entail the parallel relations specified in Eq.s~\ref{eq:f1-ref} and ~\ref{eq:parallel-constraints}.
\end{dfn}

\begin{example}\label{ex:solution_of_reference_objects}
A solution of $\net_V$ is constructed as follows (see Figure~\ref{fig:frame} for an illustration).
\begin{eqnarray}\label{eq:w-ref}
w_{ref} & = & [0,0.5]\times [0.9,1],\\ \label{eq:f-ref}
f_{ref}      & = & [0,0.5]\times [0.7,1],\\ \label{eq:fneg-ref}
f_{\neg ref} &=& [0,0.5]\times [0.4,1],\\ \label{eq:f0-ref}
f_{ref}^0  &=& [0,0.5]\times [0.2,1].
\end{eqnarray}
For each propositional variable $p_i$, we define $f_i,f_{\neg i}$, and $f_i^0$ as follows.
\begin{eqnarray}
      f_i &=& [i,i+0.3]\times[0.7,1]\\
      f_{\neg i} &=& [i,i+0.6]\times[0.4,1]\\
      f_i^0 &=& [i,i+0.8]\times[0.2,1]
\end{eqnarray}

The network $\net_V$ does not impose new constraints to $u_i$ and $u_{\neg i}$. Therefore, we can lay $u_i$ horizontally and $u_{\neg i}$ vertically, or vice versa. For example, we may define (see $u_1$ and $u_{\neg 1}$ in Figure~\ref{fig:frame} for illustration).
\begin{equation}
u_i=[i,i+0.5]\times[0.8,1],\quad u_{\neg i}=[i,i+0.4]\times[0.3,1],
\end{equation}
or vise versa (see $u_i$ and $u_{\neg i}$ in Figure~\ref{fig:frame} for illustration),
\begin{equation}
u_i=[i,i+0.2]\times[0.5,1],\quad u_{\neg i}=[i,i+0.7]\times[0.6,1].
\end{equation}
\end{example}

\subsection{CDC Constraints Related to Clauses}

In the above subsection, we have set up the correspondence between the truth value (\true/\false) of a propositional variable $p$ and the vertical/horizontal state of the corresponding spatial variable $u_p$. This subsection introduces for each clause $c\equiv p_r^\ast \vee p_s^\ast \vee p_t^\ast$ ($1\leq r<s<t\leq n$) a basic CDC constraint network $\net_c$. For each truth assignment $\pi:V\rightarrow \{\true,\false\}$, we prove that ``$\pi$ satisfies $c$" is equivalent to that ``$\net_c$ has a solution in which $u_i$ is vertically instantiated if and only if $\pi(p_i)=\true$ for $i=r,s,t$". Write $u_i^\ast$ for the spatial variable that corresponds to $p_i^\ast$, i.e.
\begin{equation} \label{eq:u_i^star}
u^\ast_i = \left\{
   \begin{array}{l l}
\mbox{$u_i$}, & \mbox{if $p_i^\ast=p_i$,}\\
\mbox{$u_{\neg i}$}, & \mbox{if $p_i^\ast=\neg p_i$.}\\
  \end{array}
  \right.
 \end{equation}
Assume, moreover, a vertical/horizontal state $state_i$ takes value in $\{$vertical, horizontal$\}$. The above equivalence statement means that, for each 3-tuple $(state_r,state_s,state_t)$ of vertical/horizontal states,  $\net_c$ has a solution in which $u_i^\ast$ is in $state_i$ for $i=r,s,t$ if and only if $state_r,state_s,state_t$ are not all horizontal.

\begin{figure}[htbp]
\centering
  \includegraphics[width=.8\textwidth]{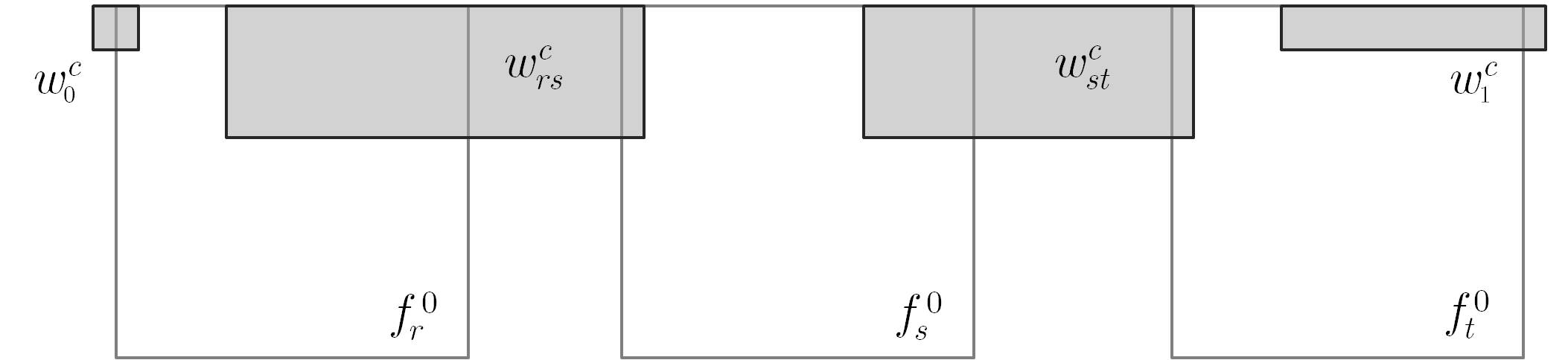}\\
  \caption{Positions of $w^c_0,w^c_{rs},w^c_{st},w^c_1$.}\label{fig:w1}
\end{figure}

We next give an intuitive explanation for the construction of constraints in $\net_c$. Consider the frames of $p_r$, $p_s$, and $p_t$. We introduce four auxiliary spatial variables $w_0^c,w_{rs}^c,w_{st}^c$ and $w_1^c$ such that they are bridged by $f_r^0$, $f_s^0$, and $f_t^0$ in the sense that their $x$-projections are overlapped one by one in the ordering
$w_0^c,f_r^0,w_{rs}^c,f_s^0,w_{st}^c,f_t^0,w^c_1$ (see Figure~\ref{fig:w1}). The spatial variables $u_i^\ast$ ($i=r,s,t$) may be either horizontally or vertically instantiated. To exclude the case where  $u_r^\ast$, $u_s^\ast$, and $u_t^\ast$ are all horizontally instantiated, we introduce a new spatial variable $v_c$ and several new constraints. Intuitively, the mbr of $v_c$ has the form as shown in Figure~\ref{fig:w3}, but the interior of $v_c$ is disjoint from spatial variables in
\begin{equation}\label{eq:Xc}
X_c \equiv \{w_0^c,u_r^\ast,w_{rs}^c,u_s^\ast,w_{st}^c,u_t^\ast, w_1^c\}.
\end{equation}
Note this is possible only if  $u_r^\ast$, $u_s^\ast$, and $u_t^\ast$ do not bridge all the gaps between $w_0^c,w_{rs}^c,w_{st}^c$ and $w_1^c$. In what follows, we refer to this as the \emph{gap condition}.

The gap condition is fulfilled by imposing the following constraints:
\begin{align}
\label{eq:gap-condition-1}
& x\ O \ v_c && && (x\in X_c),\\
\label{eq:gap-condition-2}
& v_c \ E\!:\!SE\!:\!S\ w_0^c, && v_c\ S\!:\!SW\!:\!W\ w_1^c\\
\label{eq:gap-condition-3}
& v_c\ E\!:\!SE\!:\!S\!:SW\!:\!W\ x && && (x\in\{u_r^\ast,w_{rs}^c,u_s^\ast,w_{st}^c,u_t^\ast\}).
\end{align}

Note that the constraint of $v_c$ to a spatial variable $x$ in $X_c$ does not contain tile name $O$. This means that the interior of $v_c$ is disjoint from $x$. From constraints $w_0^c\ O\ v_c$ and $v_c\ E\!:\!SE\!:\!S\ w_0^c$, we know $\mbr(v_c)$ has the same upper left corner point as $\mbr(w_0^c)$ does. Similarly, $\mbr(v_c)$ has the same upper right corner point as $\mbr(w_1^c)$ does.

Assume the gap condition is violated. This means, there is no gap between any two consecutive regions in $X_c$ (cf. Figure~\ref{fig:w2}). In this case, the union of these regions contains the  rectangle $I_x(v_c)\times I_y(w_0^c)$, which should be excluded from the interior of $v_c$. This contradicts the requirement that $\mbr(v_c)$ shares the same upper left corner point with $w_0^c$. Therefore, the constraints are not satisfiable.

\begin{figure}[htbp]
\centering
  \includegraphics[width=.8\textwidth]{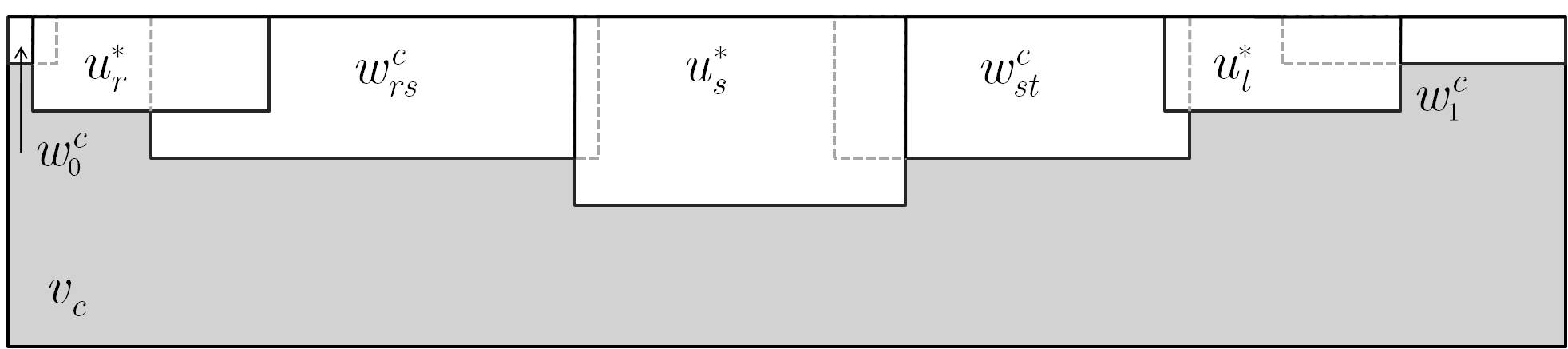}\\
  \caption{Illustration of the case that the gap condition is violated.}\label{fig:w2}
\end{figure}

On the other hand, suppose there is a gap between two consecutive regions in $X_c$ (cf. Figure~\ref{fig:w3}). It is straightforward to check that all the constraints are satisfied if we let $v_c$ be the region obtained from (after necessary regularization) subtracting regions in $X_c$ from the rectangle $\mbr(v_c)$.

\begin{figure}[htbp]
\centering
  \includegraphics[width=.8\textwidth]{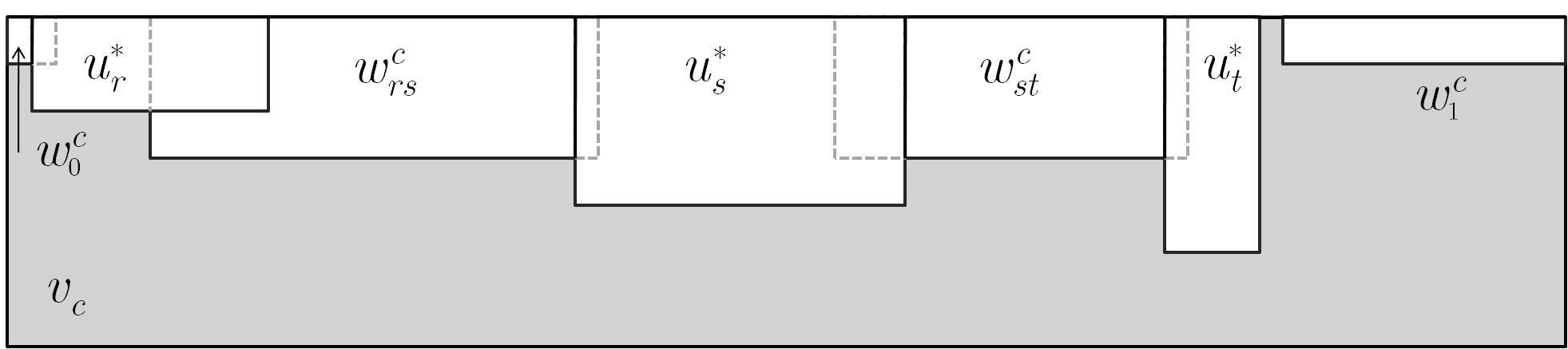}\\
  \caption{Illustration of $v_c$ when the gap condition is satisfied.}\label{fig:w3}
\end{figure}

After an intuitive description, we next introduce the basic CDC constraints in $\net_c$ for clause $c\equiv p^\ast_r\vee p^\ast_s\vee p^\ast_t$ in $\phi$.

We begin with the constraints involving the four auxiliary `pier' spatial variables $w_0^c,w_{rs}^c,w_{st}^c$ and $w_1^c$. As shown in Figure~\ref{fig:w1}, these variables are interpreted as rectangles that are bridged by the frames $f_r^0$, $f_s^0$, and $f_t^0$.

\begin{eqnarray}\label{eq:wc-1}
&&\parallel (w^c_0,w_{ref}), \parallel (w^c_1,w_{ref}),\\
\label{eq:wc0}
&&\sfo\otimes\sff(w^c_0,f_r),\\
\label{eq:w^r_c}
&&\left\{
\begin{array}{l l}
  \sfo\otimes\sfeq(f_r,w^c_{rs}), \sfo\otimes\sfeq(w^c_{rs}, f_s), & \mbox{if $p_r^\ast=p_r$,}\\
  \sfo\otimes\sffi(f_{\neg r}, w^c_{rs}), \sfo\otimes\sfeq(w^c_{rs},f_s), & \mbox{if $p_r^\ast=\neg p_r$,}
\end{array}
\right. \\
\label{eq:w^s_c}
&&\left\{
\begin{array}{l l}
  \sfo\otimes\sfeq(f_s,w^c_{st}), \sfo\otimes\sfeq(w^c_{st},f_t), & \mbox{if $p_s^\ast=p_s$,}\\
  \sfo\otimes\sffi(f_{\neg s}, w^c_{st}), \sfo\otimes\sfeq(w^c_{st},f_t), & \mbox{if $p_s^\ast=\neg p_s$,}
\end{array}
  \right. \\
\label{eq:w^t_c}
&&\left\{
\begin{array}{l l}
  \sfo\otimes\sffi(f_t,w^c_1), & \mbox{if $p_t^\ast=p_t$,}\\
  \sfo\otimes\sffi(f_{\neg t}, w^c_1), & \mbox{if $p_t^\ast=\neg p_t$.}
\end{array}
\right.
\end{eqnarray}
Note the $\parallel$ constraint and RA constraints $\sfo\otimes\sff$, $\sfo\otimes\sfeq$, $\sfo\otimes\sffi$ in the above equations are shorthands of the basic CDC constraints that entail them (cf. Example~\ref{ex:net_sf}). The first equation (Eq.~\ref{eq:wc-1}) requires that $w^c_0$ and $w^c_1$ are of the same height as the reference spatial variable $w_{ref}$. The second equation (Eq.~\ref{eq:wc0}) specifies the RA relation between $w^c_0$ and $f_r$. The third equation (Eq.~\ref{eq:w^r_c}) specifies that $w_{rs}^c$ is of the same height as the inner frame $f_i$ of $u_i$. The position of $w_{rs}^c$, however, depends on the sign of literal $p_r^\ast$. If $p_r^\ast$ is positive, then we require $w_{rs}^c$ to bridge the gap between $f_r$ and $f_s$; otherwise, we require $w_{rs}^c$ to bridge the gap between $f_{\neg r}$ and $f_s$. The constraints involving $w_{st}^c$, specified in Eq.~\ref{eq:w^s_c}, are similar. The last equation (Eq.~\ref{eq:w^t_c}) specifies that $f_t$ overlaps $w_1^c$ if $p_t^\ast$ is positive, and $f_{\neg t}$ overlaps $w_1^c$ otherwise.

We illustrate the construction of the above constraints with an example.
\begin{example}\label{ex:net-c}
Consider the clause $c=p_r\vee\neg p_s\vee p_t$. Figure~\ref{fig:cl3} illustrates the configuration of variables $w^c_0,w^c_{rs},w^c_{st},w^c_1$ in the frames of $u_r,u_s$ and $u_t$.
The network $\net_c$ specifies that $\sfo\otimes\sff (w^c_0,f_r)$; $\sfo\otimes\sfeq(f_r, w^c_{rs}), \sfo\otimes\sfeq(w^c_{rs},f_s)$; $\sfo\otimes\sffi(f_{\neg s},w^c_{st})$, $\sfo\otimes\sfeq(w^c_{st},f_t)$; and $\sfo\otimes\sffi (f_t,w^c_1)$.

\begin{figure}[htb]
\centering
\includegraphics[width=.8\textwidth]{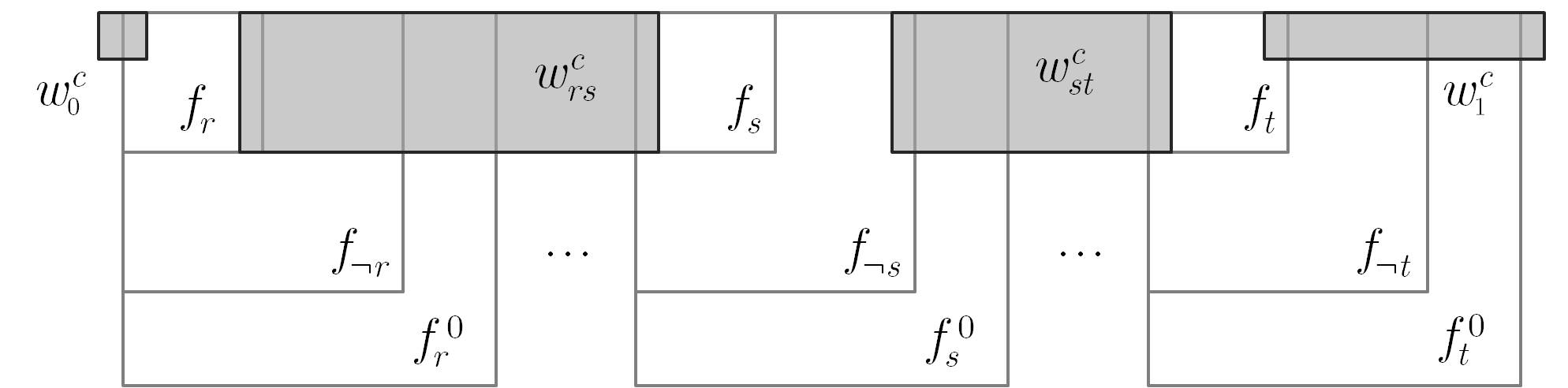}%[width=5in]
\caption{Configurations of $w^c_0,w^c_{rs},w^c_{st},w^c_1$ for clause $c=p_r\vee\neg p_s\vee p_t$.}
\label{fig:cl3}
\end{figure}

Figure~\ref{fig:cl2} examines the possible position of $u_r^\ast=u_r$ and $u_s^\ast=u_{\neg s}$.
\begin{figure}[htb]
\centering
\begin{tabular}{cc}
\includegraphics[width=.4\textwidth]{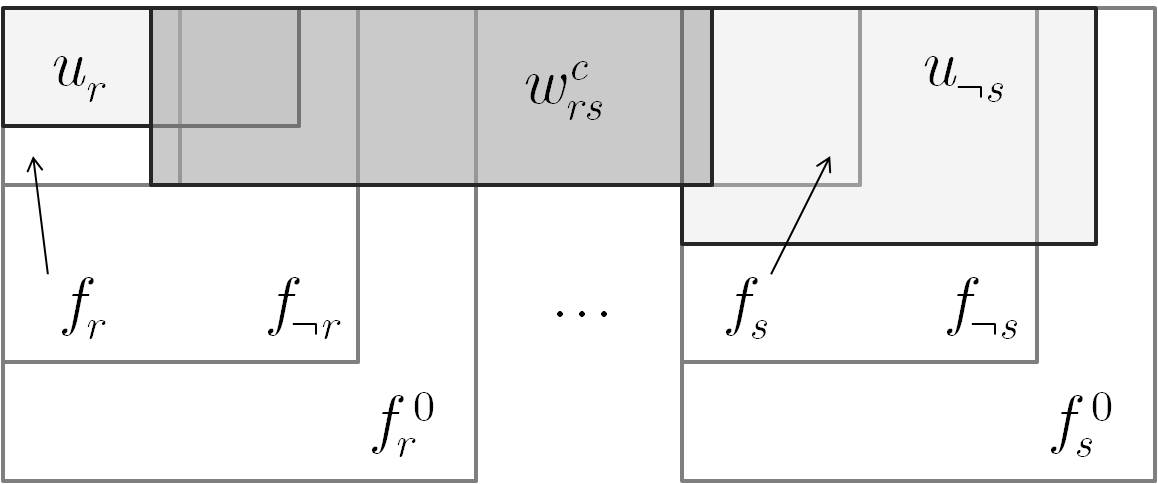}%[width=5in]
&
\includegraphics[width=.4\textwidth]{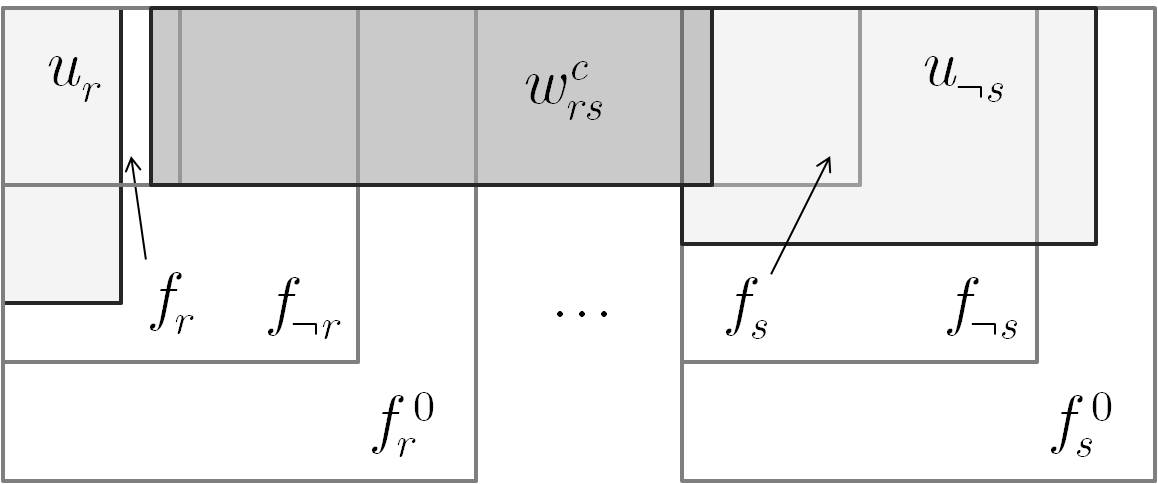}%[width=5in]
\\
(a) & (b)\\
\includegraphics[width=.4\textwidth]{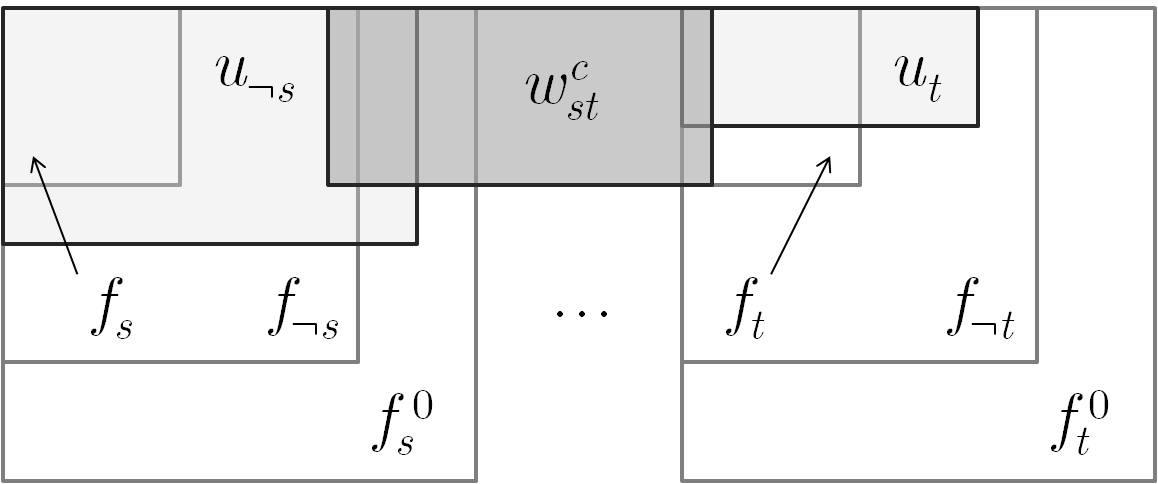}%[width=5in]
&
\includegraphics[width=.4\textwidth]{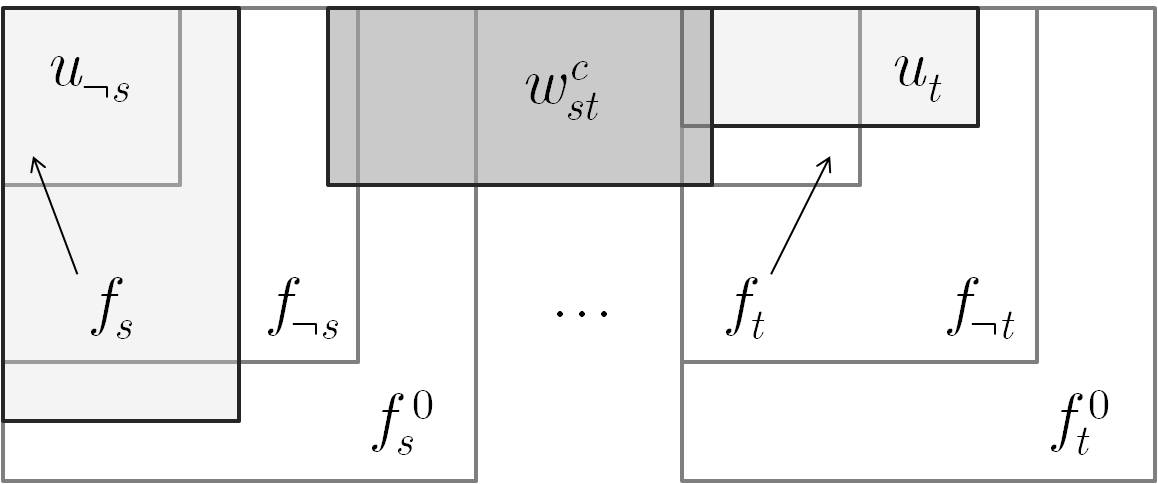}%[width=5in]
\\
(c) & (d)
\end{tabular}
\caption{Possible configurations of $u_r$ and $u_{\neg s}$: (a) $u_r$ is horizontally instantiated; (b) $u_r$ is vertically instantiated; (c) $u_{\neg s}$ is horizontally instantiated; (d) $u_{\neg s}$ is vertically instantiated.}
\label{fig:cl2}
\end{figure}
If $u_r$ is horizontally instantiated, then the gap between $u_r$ and $u_{\neg s}$ is certainly bridged by $w^c_{rs}$ (Figure~\ref{fig:cl2}(a)); if $u_r$ is vertically instantiated, then it is possible to make $u_r$ `thin' enough so that the gap between $u_r$ and $u_{\neg s}$ is maintained  (Figure~\ref{fig:cl2}(b)). Similar results hold for $u_{\neg s}$ (see Figure~\ref{fig:cl2}(c)(d) for illustration). In case $u_r$, $u_{\neg s}$, and $u_t$ are all horizontally instantiated, then there is no gap between any consecutive two of the seven regions. Otherwise, if any of $u_r$, $u_{\neg s}$, and $u_t$ is vertically instantiated, then it is possible to maintain some gap.
\end{example}

Combining with the constraints involving the spatial variable $v_c$, we are now ready to introduce $\net_c$.
\begin{dfn}\label{dfn:net_c}
The basic CDC network $\net_c$ for $c=p_r^\ast \vee p_s^\ast \vee p_t^\ast$ contains the basic CDC constraints in $\net_V$ (see Definition~\ref{dfn:net_frame}), and the basic CDC constraints used, explicitly or implicitly, in Eq.s~\ref{eq:gap-condition-1}-\ref{eq:w^t_c}.
\end{dfn}

Note two new parallel relations are introduced in $\net_c$. The spatial variable set of $\net_c$ includes those in $\net_V$, and $v_c,w^c_0,w^c_{rs},w^c_{st},w^c_1$, and two auxiliary variables for constructing parallel relations.

\begin{prop}\label{prop:bridge}
Suppose $\net_c$ is the basic CDC network defined for clause $c\equiv p_{r}^\ast\vee p_s^\ast \vee p_t^\ast$.
In any solution of $\net_c$, if $u^\ast_r$ ($u^\ast_s$, $u^\ast_t$, resp.) is horizontally instantiated,
then its mbr bridges the gap between $\mbr(w^c_0)$ ($\mbr(w^c_{rs})$, $\mbr(w^c_{st})$, resp.) and $\mbr(w^c_{rs})$ ($\mbr(w^c_{st})$, $\mbr(w^c_1)$, resp.).
\end{prop}

\begin{proof}
Recall that if $u_r$ is horizontally instantiated, then $I_x(f_r)\ \sfs\ I_x(u_r)\ \sfs\ I_x(f_{\neg r})$ and
$I_x(f_r)\ \sfs\ I_x(u_{\neg r})\ \sfs\ I_x(f_{\neg r})$; if $u_{\neg r}$ is horizontally instantiated, then $I_x(u_r)\ \sfs\ I_x(f_r)$
and $I_x(f_{\neg r})\ \sfs\ I_x(u_{\neg r})\ \sfs\ I_x(f_r^0)$ (cf. Figure~\ref{fig:frame}). Furthermore,
note that the top edges of the mbrs of these regions are on the same line. It is easy to see that $\mbr(u_r)$ ($\mbr(u_{\neg r})$, resp.)
bridges the gap between $\mbr(w_0^c)$ and $\mbr(w_{rs}^c)$ if $u_r$ ($u_{\neg r}$, resp.) is horizontally instantiated and $p_r$ is positive
(negative, resp.) in the clause. Note that $u_r^\ast=u_r$ ($u_r^\ast=u_{\neg r}$, resp.) if $p_r$ is positive (negative, resp.) in the clause.
The proposition follows directly.
\end{proof}

We now show that $\net_c$ has the following property.

\begin{prop}\label{prop:netc}
Suppose $\net_c$ is the basic CDC network defined for clause $c=p_r^\ast \vee p_s^\ast \vee p_t^\ast$ ($r<s<t$). Assume, moreover,  $\pi:V\rightarrow \{\true,\false\}$ is a truth assignment. Then the following statements are equivalent:
\begin{itemize}
\item $\pi$ satisfies $c$, that is, at least one of the three literals $p_r^\ast$, $p_s^\ast$, and $p_t^\ast$ is $\true$ under $\pi$;
\item $\net_c$ has a solution in which $u_i$ is vertically instantiated  iff $\pi(p_i)=\true$ for $i=r,s,t$.
\end{itemize}
\end{prop}
\begin{proof}
We first prove that, if $\pi$ does not satisfy $c$, then $\net_c$ has no solution in which $u_i$ is vertically instantiated iff $\pi(p_i)=\true$ for $i=r,s,t$.
Or equivalently, if $\pi(p^\ast_i)=\false$ for $i=r,s,t$, then $\net_c$ has no solution in which $u^\ast_i$ is horizontally instantiated for $i=r,s,t$.
We prove this statement by contradiction. Suppose $\net_c$ has a solution in which $u_i^\ast$ is horizontally instantiated for $i=r,s,t$. Then, by Proposition~\ref{prop:bridge}, $u_r^\ast$ bridges the gap between $\mbr(w^c_0)$ and $\mbr(w^c_{rs})$; $u_s^\ast$ bridges the gap between $\mbr(w^c_{rs})$ and $\mbr(w^c_{st})$; and $u_t^\ast$ bridges the gap between $\mbr(w^c_{st})$ and $\mbr(w^c_1)$ (cf. Figure~\ref{fig:w2}). Let
$$A_c=\mbr(v_c)\setminus (\bigcup\{\mbr(x):x\in X_c\}),$$
where $X_c$ is defined as in Eq.~\ref{eq:Xc}.
It is clear that $\mbr(A_c)$ is a proper subset of $\mbr(v_c)$. Because the interior of $v_c$ is disjoint from $\mbr(x)$ ($x\in X_c$), we have $v_c\subseteq A_c$. This leads to a contradiction. Therefore, if $\pi$ does not satisfy $c$, then $\net_c$ has no solution in which $u_i$ is vertically instantiated iff $\pi(p_i)=\true$ for $i=r,s,t$.

On the other hand, suppose $\pi$ satisfies $c$. We construct a solution of $\net_c$ in which $u_i$ is vertically instantiated iff $\pi(p_i)=\true$ for $i=r,s,t$. Variables in $\net_V$ other than $u_i,u_{\neg i}\ (1\leq i\leq n)$ are defined as the same in Example~\ref{ex:solution_of_reference_objects}.

For each propositional variable $p_i$, we define $u_i$ and $u_{\neg i}$ as follows (cf. Figure~\ref{fig:var}).
\begin{eqnarray}\label{eq:u_i}
 u_i = \left\{
   \begin{array}{l l}
  \mbox{$[i,i+0.2]\times [0.5,1]$}, & \mbox{if $\pi(p_i)$ is $\true$},\\
  \mbox{$[i,i+0.5]\times [0.8,1]$},  & \mbox{otherwise.}\\
  \end{array}
  \right.\\
\label{eq:u_i^prime}
 u_{\neg i} = \left\{
   \begin{array}{l l}
  \mbox{$ [i,i+0.7]\times [0.6,1]$},  & \mbox{if $\pi(p_i)$ is $\true$},\\
  \mbox{$[i,i+0.4]\times [0.3,1]$}, & \mbox{otherwise.}\\
  \end{array}
  \right.
\end{eqnarray}
For clause $c\equiv p_r^\ast \vee p_s^\ast \vee p_t^\ast$, we define the bridge variables as follows.
\begin{align}\label{eq:w_c^0}
w^c_0&= [r-0.05,r+0.05]\times [0.9,1],\\
\label{eq:w_c^r}
w^c_{rs}&= \left\{
\begin{array}{l l}
  \mbox{$[r+0.25,s+0.05]\times [0.7,1]$}, & \mbox{if $p^\ast_r=p_r$,}\\
  \mbox{$[r+0.55,s+0.05]\times [0.7,1]$},  & \mbox{otherwise.}\\
\end{array}
\right.\\
\label{eq:w_c^s}
w^c_{st}&= \left\{
\begin{array}{l l}
  \mbox{$[s+0.25,t+0.05]\times [0.7,1]$}, & \mbox{if $p^\ast_s=p_s$,}\\
  \mbox{$[s+0.55,t+0.05]\times [0.7,1]$},  & \mbox{otherwise.}\\
\end{array}
\right.\\
\label{eq:w_c^t}
w^c_1&= \left\{
\begin{array}{l l}
  \mbox{$[t+0.25,t+0.85]\times [0.9,1]$}, &  \mbox{if $p^\ast_t=p_t$,}\\
  \mbox{$[t+0.55,t+0.85]\times [0.9,1]$},  & \mbox{otherwise.}\\
\end{array}
\right.
\end{align}

It is straightforward to verify that the three gaps between $w^c_0,w^c_{rs},w^c_{st},w^c_1$ are all bridged iff $u^\ast_r,u^\ast_s,u^\ast_t$ are all horizontally instantiated. Let $v_c$ be the region obtained by subtracting from $[r-0.05,t+0.85]\times [0,1]$ the union of $x$ ($x\in X_c$). Since the values of $p_r^\ast$, $p_s^\ast$, and $p_t^\ast$ are not all $\false$, we know that at least one of the gaps between the bridge variables are maintained. This guarantees that the above instantiation of $v_c$ satisfies all constraints in $\net_c$ involving $v_c$. Therefore, we have constructed a solution of $\net_c$ with the desired property.
\end{proof}

As a corollary, we know in particular that $\net_c$ is consistent, and at least one of $u_r^\ast$, $u_s^\ast$, and $u_t^\ast$ is vertically instantiated in any solution of $\net_c$.

\begin{dfn}\label{dfn:net_phi}
For a 3-SAT instance $\phi=\bigwedge_{j=1}^m c_j$ over $V=\{p_1,\cdots,p_n\}$, we define $\net_{c_j}$ as the basic CDC network for clause $c_j$ as in Definition~\ref{dfn:net_c}, and define $\net_\phi$ as the (incomplete) basic CDC network that is the union of all $\net_{c_j}$ ($1\leq j\leq m$).
\end{dfn}

We next show $\phi$ is satisfiable if $\net_\phi$ is satisfiable.
\begin{lemma}\label{coro:nece}
Let $\phi$ be a 3-SAT instance and let $\net_\phi$ be the basic CDC network of $\phi$.
If $\net_\phi$ is satisfiable, then $\phi$ is also satisfiable.
\end{lemma}
\begin{proof}
Let $\sa$ be a solution of $\net_\phi$. Define a truth assignment $\pi:\{p_1,\cdots,p_n\}\rightarrow \{\true,\false\}$ as: $\pi(p_i)=\true$ if and only if $u_i$ is vertically instantiated in $\sa$. For each clause $c$ of $\phi$, we know $\sa$ is also a solution of $\net_c$. By definition of $\pi$, we know in particular that $u_i$ is vertically instantiated in $\sa$ if and only if $\pi(p_i)$ is true for $i=r,s,t$.
By Proposition~\ref{prop:netc}, $\pi$ satisfies $c$. Due to the arbitrariness of $c$, we know $\pi$ satisfies $\phi$. Therefore, $\phi$ is satisfiable.
\end{proof}
On the other hand, we show $\net_\phi$ is satisfiable only if $\phi$ is satisfiable.

\begin{lemma}\label{lemma:sufficiency}
Let $\phi$ be a 3-SAT instance and let $\net_\phi$ be the basic CDC network of $\phi$.
If $\phi$ is satisfiable, then $\net_\phi$ is also satisfiable.
\end{lemma}
\begin{proof}
Suppose $\pi:V \rightarrow \{\true,\false\}$ is a truth assignment. A solution for $\net_\phi$ can be constructed by following exactly the same procedures as we have used in Proposition~\ref{prop:netc}.
Note that there are no direct constraints between variables in $X_c$ and $X_c^\prime$, where $c,c^\prime$ are two different clauses. We instantiate $v_c$ as in Proposition~\ref{prop:netc}. It is easy to see that the assignment satisfies $\net_c$ for each clause $c$ of $\phi$. Therefore, the network $\net_\phi$ is also satisfiable.
\end{proof}

As a consequence of the above results, we have

\begin{thm}\label{thm:np-hard}
Deciding the consistency of a possibly incomplete basic CDC network is NP-hard.
\end{thm}
\begin{proof}
We prove the NP-hardness of consistency checking of basic CDC networks by a reduction from 3-SAT. For each 3-SAT instance $\phi$, we define an incomplete basic CDC network $\net_\phi$.
Lemmas~\ref{coro:nece} and ~\ref{lemma:sufficiency} show that the 3-SAT instance $\phi$ is satisfiable if and only if the basic CDC network $\net_\phi$ is satisfiable. It is not hard to show that the total variables in $\net_\phi$ is linear to the total number of variables and clauses of $\phi$. This shows that the size of $\net_\phi$ is polynomial of the size of $\phi$. So we have reduced 3-SAT in polynomial time to the consistency problem of possibly incomplete basic CDC networks.
\end{proof}

To determine the consistency of a possibly incomplete basic CDC network, we non-deterministically replace all unspecified constraints with basic constraints and then determine the consistency of the complete basic CDC network by the cubic time algorithm introduced in \cite{Liu+2010}. This implies that the consistency problem of basic CDC networks is in NP. As a corollary of our main theorem, we have
\begin{coro}
Deciding the consistency of a possibly incomplete basic CDC network is NP-Complete.
\end{coro}

\section{The Reduction to the Consistency Problem in CDC$_d$}
The CDC is defined for connected regions. Allowing regions to be disconnected, we obtain a variant of CDC, written CDC$_d$ in this paper. Is the reduction described in Section~4 applicable to CDC$_d$? The answer is yes! This is because, in the reduction, we do not use the connectedness property at all.  Note each basic CDC relation is contained in the corresponding basic CDC$_d$ relation. All basic constraints used in the reduction are also representable in CDC$_d$. Moreover, all definitions and results obtained in Section~3 can be applied to CDC$_d$. Therefore, we have

\begin{thm}\label{thm:cdc-d}
Deciding the consistency of a possibly incomplete basic CDC$_d$ network is NP-hard.
\end{thm}

This suggests that the $O(n^5)$ algorithm devised in \cite{SkiadopoulosK05} for determining the consistency of basic CDC$_d$ networks is incomplete.

Suppose $\net$ is a possibly incomplete basic CDC$_d$ network. Algorithm \consistency\  \cite{SkiadopoulosK05} first transforms constraints in $\net$ into a network $O$ of Point Algebra (PA) constraints (which may also be incomplete). It then calls the \cspan\ algorithm of van Beek \cite{vanBeek1992} to compute a solution of $O$ and transforms the solution of $O$ into a maximal solution (cf. \cite{SkiadopoulosK05} for the definition). The algorithm then returns `consistency' if this particular maximal solution satisfies the NTB property, and returns `inconsistency' otherwise.

Because the network of constraints may be incomplete, the PA network $O$ may have exponentially many different (maximal) solutions.\footnote{Two solutions of a PA network are regarded as different if the orderings of the points in the two solutions are different.}  As a polynomial algorithm, Algorithm \cspan\ returns only one solution of $O$. It has been proved \cite[Theorem 3]{SkiadopoulosK05} that $O$ is consistent if and only if it has \emph{a} maximal solution which satisfies the NTB property. It is very likely that \emph{some} maximal solutions of $O$ satisfy the NTB property, while others do not. So if we want to assure $O$ is inconsistent, we need to try all different maximal solutions of $O$, which may take exponential time. Algorithm \consistency, however, checks this for only one maximal solution (constructed on the result of Algorithm \cspan). This explains why it is an incomplete algorithm for checking the consistency of basic CDC$_d$ networks.

Take the inconsistency basic network in \cite[Example 13]{SkiadopoulosK05} as an example. This network is defined as
\begin{equation}
\net=\{x\ N\!:\!E\!:\!O\ y, x\ O\!:\!S\!:\!W\ z, y\ SW\ z\}.
\end{equation}

The inconsistency of $\net$ is detected by Algorithm \cspan. In fact, the algorithm first transform $\net$ into a set $O$ of PA constraints, and then \cspan\ returns a solution of $O$, and then uses this solution to compute a maximal solution of $O$. Write $\mathfrak{m}$ for this maximal solution. The algorithm then returns `inconsistency' after showing that $\mathfrak{m}$ does not satisfy the NTB property. So far so good. Let $\net^\prime$ be the network obtained by removing the third constraint $y\ SW\ z$ from $\net$. It is easy to see that $\net^\prime$ is consistent.  For this network, a subset $O^\prime$ of $O$ is computed, and it is likely that the algorithm also takes $\mathfrak{m}$ as a maximal solution of $O^\prime$. If this is the case, the algorithm will return `inconsistency' for $\net^\prime$ because $\mathfrak{m}$ does not satisfy the NTB property. This is, however, incorrect.

\section{Conclusion}
In this paper, we have proved that deciding the consistency of basic but possibly incomplete CDC networks is an NP-hard problem. Combined with the tractable result reported in \cite{Liu+2010}, this draws a sharp boundary between the tractable and intractable subclasses of the CDC. It seems that the CDC is the first known qualitative calculus in which reasoning with conjunctive constraints is NP-hard, while reasoning with explicit constraints is in P. Our result is achieved by using a polynomial reduction from 3-SAT, which is also applied to CDC$_d$, the cardinal direction calculus for possibly disconnected regions. This suggests that the $O(n^5)$ algorithm in \cite{SkiadopoulosK05} is incomplete for checking the consistency of basic CDC networks. Future work will consider approximating methods for solving the consistency decision problem in the CDC.

\bibliographystyle{plain}
\bibliography{cdc-Hardness}

\end{document}